\documentclass{article}

\pdfoutput=1 

\newcommand{\ignore}[1]{}

\usepackage{graphicx}
\usepackage{epstopdf}
\usepackage{amsmath}
\usepackage{amssymb}
\usepackage{hhline}
\usepackage{multicol}
\usepackage{fancyvrb}
\usepackage{enumitem}

\usepackage{fancyvrb}

\newcommand{\boxtheorem}{\hfill $\Box$}
\newcommand{\nit}[1]{{\it #1}}
\newcommand{\mc}[1]{\mathcal{ #1}}

\newcommand{\red}[1]{\textcolor{red}{#1}}

\newcommand{\e}{\mathbf{e}}

\newcommand{\comlb}[1]{{\vspace{2mm}\noindent \bf  {\red{COMM(LEO):}}}~ #1 \hfill {\bf
    END.}\\}

\usepackage{multirow}

\usepackage{amsthm}
\usepackage{authblk}

\usepackage{array}
\newcolumntype{L}[1]{>{\raggedright\let\newline\\\arraybackslash\hspace{0pt}}p{#1}}
\newcolumntype{C}[1]{>{\centering\let\newline\\\arraybackslash\hspace{0pt}}p{#1}}
\newcolumntype{R}[1]{>{\raggedleft\let\newline\\\arraybackslash\hspace{0pt}}p{#1}}

\usepackage{tikz,tikz-qtree,tikz-qtree-compat,tikz-3dplot}                      

\usetikzlibrary{shapes,decorations.markings,arrows.meta,fit,calc,positioning}

\usepackage{url}
\usepackage{bm}

\colorlet{LightViolet}{violet!40}
\colorlet{LightRed}{red!40}
\colorlet{LightOrange}{orange!40}
\colorlet{LightGreen}{green!40}
\colorlet{LightBlue}{blue!40}
\colorlet{DarkGreen}{green!50!black}
\colorlet{DarkRed}{red!70!black}
\colorlet{DarkCyan}{red!70!black}
\colorlet{DarkBlue}{blue!80!black}
{\definecolor{DarkOrange}{rgb}{1.0, 0.49, 0.0}
\definecolor{Airforceblue}{rgb}{0.36, 0.54, 0.66}

\newcommand{\defeq}{\stackrel{\mathrm{def}}{=}}

\newcommand{\R}{\mathbb R} 

\newcommand{\E}{\mathbf{E}}

\newcommand{\be}{\begin{enumerate}}
\newcommand{\ee}{\end{enumerate}}
\newcommand{\bi}{\begin{itemize}}
\newcommand{\ei}{\end{itemize}}
\newcommand{\beq}{\begin{equation}}
\newcommand{\eeq}{\end{equation}}

\newcommand{\bp}{\begin{proof}}
\newcommand{\ep}{\end{proof}}
\newcommand{\bcor}{\begin{cor}}
\newcommand{\ecor}{\end{cor}}
\newcommand{\bthm}{\begin{thm}}
\newcommand{\ethm}{\end{thm}}
\newcommand{\blmm}{\begin{lmm}}
\newcommand{\elmm}{\end{lmm}}
\newcommand{\bdefn}{\begin{defn}}
\newcommand{\edefn}{\end{defn}}
\newcommand{\bprop}{\begin{prop}}
\newcommand{\eprop}{\end{prop}}
\newcommand{\bconj}{\begin{conj}}
\newcommand{\econj}{\end{conj}}
\newcommand{\bopm}{\begin{opm}}
\newcommand{\eopm}{\end{opm}}
\newcommand{\bprm}{\begin{prm}}
\newcommand{\eprm}{\end{prm}}
\newcommand{\brmk}{\begin{rmk}}
\newcommand{\ermk}{\end{rmk}}
\newcommand{\bclm}{\begin{claim}}
\newcommand{\eclm}{\end{claim}}
\newcommand{\bex}{\begin{ex}}
\newcommand{\eex}{\end{ex}}

\theoremstyle{plain}                   
\newtheorem{thm}{Theorem}[section]
\newtheorem{lmm}[thm]{Lemma}
\newtheorem{prop}[thm]{Proposition}
\newtheorem{cor}[thm]{Corollary}

\theoremstyle{definition}              

\newtheorem{opm}[thm]{Open Problem}
\newtheorem{prm}[thm]{Problem}
\newtheorem{conj}[thm]{Conjecture}
\newtheorem{ex}[thm]{Example}

\newtheorem{defn}[thm]{Definition}

\newtheorem{rmk}[thm]{Remark}

\newcounter{claimCounter}
\newtheorem{claim}[claimCounter]{Claim}

\newtheorem*{example*}{Example}

\newtheorem*{lmm*}{Lemma}

\definecolor{Red}{RGB}{255,204,204}
\definecolor{Green}{RGB}{204,255,204}
\definecolor{Blue}{RGB}{204,204,255}

\newcommand{\set}[1]{\{#1\}}                    
\newcommand{\setof}[2]{\{{#1}\mid{#2}\}}        

\newcommand{\eat}[1]{}

\newcommand{\counter}{{\sc COUNTER}}
\newcommand{\resp}{{\sc RESP}}
\newcommand{\shap}{{\sc SHAP}}
\newcommand{\fico}{{\sc FICO}}

\def\counterScore{\mathrm{COUNTER}}
\def\respScore{\mathrm{RESP}}
\def\shapley{\mathrm{SHAP}}

\newtheorem{definition}[thm]{Definition}

\newtheorem{theorem}{Theorem}[section]
\newtheorem{proposition}[thm]{Proposition}

\usepackage[margin=.9in]{geometry}

\def\BibTeX{{\rm B\kern-.05em{\sc i\kern-.025em b}\kern-.08emT\kern-.1667em\lower.7ex\hbox{E}\kern-.125emX}}

\title{Causality-based Explanation of Classification Outcomes}

\author[1]{Leopoldo Bertossi\footnote{Member of the ``Millenium Institute for Foundational Research on Data (IMFD, Chile)''}}
\author[2]{Jordan Li}
\author[3]{Maximilian Schleich}
\author[3]{Dan Suciu}
\author[4]{Zografoula Vagena}

\affil[1]{Universidad Adolfo Ib\'a\~nez \& RelationalAI Inc.}
\affil[2]{Carleton University}
\affil[3]{University of Washington\\ \& RelationalAI Inc.}
\affil[4]{RelationalAI Inc.}

\date{}

\begin{document}

\maketitle 

%

\begin{abstract}
We propose a simple definition of an explanation for the outcome of a
classifier based on concepts from causality.  We compare it with
previously proposed notions of explanation, and study their
complexity.  We conduct an experimental evaluation with two real
datasets from the financial domain.
\end{abstract}


%
\ignore{
\begin{teaserfigure}
  \includegraphics[width=\textwidth]{sampleteaser}
  \caption{Seattle Mariners at Spring Training, 2010.}
  \Description{Enjoying the baseball game from the third-base seats. Ichiro Suzuki preparing to bat.}
  \label{fig:teaser}
\end{teaserfigure}
}
%

\ignore{++
\copyrightyear{2019}
\acmYear{2019}
\setcopyright{acmcopyright}
\acmConference[PODS'19]{38th ACM SIGMOD-SIGACT-SIGAI Symposium on Principles of Database Systems}{June 30-July 5, 2019}{Amsterdam, Netherlands}
\acmBooktitle{38th ACM SIGMOD-SIGACT-SIGAI Symposium on Principles of Database Systems (PODS'19), June 30-July 5, 2019, Amsterdam, Netherlands}
\acmPrice{15.00}
\acmDOI{10.1145/3294052.3322190}
\acmISBN{978-1-4503-6227-6/19/06}
++}

\maketitle

\pagestyle{plain}

\section{Introduction}\label{sec:intro}

Machine-learning (ML) models are increasingly used today in making
decisions that affect real people's lives, and, because of that, there
is a huge need to ensure that the models and their decisions are
interpretable by their human users.  Motivated by this need, there has
been a lot of interest recently in the ML community in studying {\em
  Interpretable models}~\cite{DBLP:journals/corr/abs-1811-10154}.
There is currently no consensus on what interpretability means, and no
benchmarks for evaluating
interpretability~\cite{DBLP:journals/corr/Doshi-VelezK17,DBLP:journals/cacm/Lipton18}.
The only consensus is that simpler models such as linear regression or
decision trees are considered more interpretable than complex models
like deep neural nets.  However, two general principles for
approaching interpretability have emerged in the literature that are
relevant to our paper.  The first is the idea of simplifying the
model.  Rudin~\cite{DBLP:journals/corr/abs-1811-10154} defines {\em
  explanation} as approximating a model so that it becomes
interpretable; thus, at a very high level, we will use the term {\em
  explanation} to mean a simple piece of information that helps
interpreting a model.  Citing Doshi{-}Velez and
Kim~\cite{DBLP:journals/corr/Doshi-VelezK17}\footnote{The attribute
  this quote to Lombrozo.}: {\em explanations are $\ldots$ the
  currency in which we exchanged beliefs}.  The second is the idea
that a good notion of explanation should be grounded in {\em
  causality}~\cite{pearl}.  This idea has frequently been mentioned in
the literature, but no consensus exists on how to convert it into a
formal definition.

There are two levels at which one can provide
explanations~\cite{DBLP:conf/kdd/Ribeiro0G16,DBLP:journals/corr/abs-1905-04610}.
A {\em global explanation} aims at explaining the model as a whole,
while a {\em local explanation} concerns a particular outcome, i.e.  a
single decision.  For example, consider a bank that uses a machine
learning model to decide whether to grant loans to individual
customers.  Then the global explanation concerns the entire model, for
example it explains it to a developer or an auditor, while, in
contrast, a local explanation concerns a single decision, for an
individual customer.  For example the customer applies for a loan, the
bank runs the model, and outcome is to deny the loan, and the customer
asks {\em why?} The bank needs to provide an explanation.  In this
paper we are concerned only with local explanations.  We argue that
they are of particular interest to the data management community,
because they need to be provided interactively, and they often require
processing a large amount of data, for example to compare the current
customer with the entire population of the bank's customers.

The golden standard for explanations are {\em black-box
  explanations}~\cite{DBLP:conf/kdd/Ribeiro0G16,DBLP:journals/corr/abs-1905-04610},
which are independent of the inner workings of the classifier.
LIME~\cite{DBLP:conf/kdd/Ribeiro0G16} explains an outcome by learning
an interpretable model locally around the outcome;
\shap~\cite{DBLP:conf/nips/LundbergL17,DBLP:journals/corr/abs-1905-04610}
explains an outcome by modeling it as a Shapley cooperative game and
assigning a score to each feature.
However, lacking a common benchmark for evaluating the {\em quality}
of the explanation, researchers are also considering {\em white-box}
explanations, based on the intuition that knowledge of the inner
workings of the model can help better explain a particular outcome.
For example, a successful explanation framework has been developed by
Chen et al.~\cite{DBLP:journals/corr/abs-1811-12615} for the Credit
Risk assessment problem, by co-designing the model and the
explanation.

The goal of this paper is three-fold.  First, we introduce a black-box
notion of explanation based on causality.  Our definition extends the
simplified notion of {\em actual cause} and {\em responsibility}
described in~\cite{DBLP:journals/pvldb/MeliouGMS11} to a new black-box
explanation score, called \resp.  Unlike \shap, which is grounded in
cooperative games, \resp\ is grounded in causality.  Second, we
examine the data-management challenge of computing black box
explanations.  Both \resp\ and \shap\ explanations require access to a
probability distribution over the population.  We consider two
probability spaces, the {\em product space} and the {\em empirical
  space}, and show that computing \shap\ is \#P-hard in the former,
while \resp\ is meaningless in the latter.  Our finding suggests that
future research needs to focus on designing realistic, yet tractable
probability spaces, in order to support the efficient computation of
black-box explanations.  Finally, we conduct an empirical evaluation
of the quality of black-box explanations.  We compare both \resp\ and
\shap\ to the white-box explanation for the Credit Risk
problem~\cite{DBLP:journals/corr/abs-1811-12615} and find that \resp\
is very close to the white-box explanation.  On closer inspection of
the cases where the two explanations differ reveals that the reason is
that the white-box explanation is based only on the current entity,
and ignores the population as a whole; in contrast, \resp\ is based on
causality and, thus, takes into account counterfactual outcomes, which
can be computed by examining the entire population.  We also compare
\resp\ and \shap\ on a second dataset (where no white-box explanation
is available) and measure their sensitivity to bucketization.

In this paper we only consider {\em feature-based explanations}, which rank the
features by some score, representing their importance for a particular outcome.
This is similar to previous work on
SHAP~\cite{DBLP:conf/nips/LundbergL17,DBLP:journals/corr/abs-1905-04610} and
to~\cite{DBLP:journals/corr/abs-1811-12615}.  Feature-based explanations are
attractive because they are very simple.  Other approaches have been proposed in
the literature:
Wachter et al.~\cite{wachter2017counterfactual}
  compute counterfactual explanations for entity $\e$ by solving an optimization
  problem that finds the entity $\e'$ closest to $\e$ for which
  $L(\e) \neq L(\e')$.
  LIME~\cite{DBLP:conf/kdd/Ribeiro0G16} provides as explanation a simple, linear
  model learned locally around the outcome. DeepLIFT~\cite{deeplift} computes
  importance scores for input features by comparing it with a reference entity.
  Goyal et al.~\cite{DBLP:journals/corr/abs-1907-07165} and Ghorbani et
  al.~\cite{DBLP:conf/nips/GhorbaniWZK19} introduce {\em concept-based
    explanations}, and Khanna et al.~\cite{DBLP:conf/aistats/KhannaKGK19}
  provide explanations consisting of items in the training data.
  Counterfactuals are also central to the notion of
  recourse~\cite{ustun:recourse}, which examines the feasibility of a user to
  change the decision of a model, and algorithmic
  transparency~\cite{datta2016}, which measures the degree of influence of
  inputs on outputs of decision systems.



In summary, this paper makes the following contributions:

\begin{itemize}
\item We define the explanation problem for a black-box classifier;
  Sec.~\ref{sec:problem}.
\item We introduce \resp, a black-box explanation score based on
  causality; Sec.~\ref{sec:da}
\item We describe the formal connection between \resp\ and \shap,
  showing that, while they are somewhat correlated, they are quite
  different; Sec.~\ref{sec:shapley}.
\item We establish the computational complexity of \resp\ and \shap\
  over two simple probabilistic models of the population;
  Sec.~\ref{sec:prob:space}.
\item We conduct an extensive experimental evaluation;
  Sec.~\ref{sec:rudin}.
\end{itemize}

\section{Problem Definition} \label{sec:problem}

Fix a set of features $\set{F_1, \ldots, F_n}$, with finite domains
$\nit{Dom}(F_i)$.  We assume to have a classification model denoted by
$L$. Given values $x_1, \ldots, x_n$ in the corresponding domains, the
classifier returns a value $L(x_1, \ldots, x_n) \in \set{0,1}$.  The
inner workings of the classifier are irrelevant to us.  It may be a
decision tree, or a boosted model, or a deep neural network, or any
other program that, given input features, returns a 0 or a 1.  We will
only assume to have access to a black-box computing $L$, given values
of the features.  Our discussion also applies to the case when $L$
returns a continuous value in $\R$, but for simplicity we will assume
its value is in $\set{0,1}$.

We consider the scenario in which $L$ is applied to an entity $\e$,
for example an individual customer petitioning for a loan.  We
identify the entity by its features,
$\e = \langle x_1, \ldots, x_n\rangle$, and denote the classifier's
output by $L(\e)$.  Our goal is to provide an explanation for this
output $L(\e)$.  Throughout the paper we will assume that the output
$L=0$ represents the ``good'' outcome while $L=1$ represents the
``bad'' outcome, for example, $L(\e)=0$ means that the loan is
granted, while $L(\e)=1$ means the loan is denied.  We are mainly
interested in explaining the outcome $L(\e)=1$, but our discussion
also applies to entities with the good outcome 0.  We consider only
{\em feature-based} explanations, which assign a score to each feature
$F_i$, and return as explanation the feature with the highest score,
or a small set of features with highest scores.  For example, if the
explanation is the feature $\texttt{NumberOfInquiresLast7Days}$, then
we would tell the customer ``your loan was denied because of the
number of inquires to your credit history in the last 7 days''.  Banks
often deny loans when there are a large number of recent inquires to
the customers' credit history, in order to prevent customers from
quickly opening multiple credit accounts at independent banks.

In order to produce meaningful explanations, we will assume that, in
addition to the entity $\e$ and the black box classifier $L$,
we also have access to a probability space $\Omega$ on entities. For
example, this can be the known distribution of the population of
customers applying for loans.  A meaningful explanation should be
informed by typical customers over that population, as we explain
below.  We will write $\e \sim \Omega$ to denote the fact that $\e$ is
chosen at random from the space $\Omega$.


In this paper we study two black-box explanation scores.  The first is
our new proposal, called the \resp-score, and is grounded in the
principle of {\em actual cause}.  The second score, \shap-score, has
been proposed recently~\cite{DBLP:conf/nips/LundbergL17}, and is based
on the Shapley value of a cooperative game.  Both are black-box
explanations, in that they are oblivious to the inner workings of the
classifier $L$.  We evaluate empirically these two explanations, and
compare them to a white-box explanation defined by Chen et
al.~\cite{DBLP:journals/corr/abs-1811-12615}, for a specific
classification task.

Throughout the paper we denote the set of $n$ features by $\mc{F}$.
For any subset $S \subseteq \mc{F}$ we denote by $\e_S$ the
restriction of an entity $\e$ to the features in $S$.


\section{The \counter- and \resp-Scores}\label{sec:da}


We introduce here our novel black-box explanation score, by applying
the principle of causal reasoning.  Halpern and
Pearl~\cite{DBLP:journals/corr/abs-1301-2275} give a formal definition
of an actual cause, while Chockler and Halpern extend it to a degree of
causality, called
{\em responsibility}~\cite{DBLP:journals/jair/ChocklerH04}.  Both notions
were simplified and adapted to database
queries~\cite{DBLP:journals/pvldb/MeliouGMS11}.  We review them here,
then adapt them to feature-based explanations.

Fix an entity $\e^\star$, and assume its outcome is $L(\e^\star)=1$.
A {\em counterfactual
  cause}~\cite{DBLP:journals/corr/abs-1301-2275,DBLP:journals/pvldb/MeliouGMS11}
is a feature $F_i$ and a value $v$ such that $L(\e)=0$, where
$\e \defeq \e^\star[F_i := v]$ is the entity obtained from $\e^\star$
by setting $F_i = v$ and keeping all other values unchanged.  Then,
$(F_i,v)$ is a counterfactual cause, if, by changing only $F_i$ to
$v$, the outcome changes from $1$ to $0$, i.e. from ``bad'' to
``good''. Thus, the notion of a counterfactual applies to the
{\em pair} $(F_i, v)$, i.e. the value is important.

We adapt this notion to define the \counter-score of a feature $F_i$
(without any value).  The \counter-score is defined as the expected
counterfactual causality over the random choices of the values $v$:

\begin{definition} \label{def:counter}
  Fix an entity $\e^\star$.  The \counter-score of a feature $F_i$ is:
  \begin{align*}
    \counterScore(\e^\star,F_i) \defeq &{} L(\e^\star) - \E\left[L(\e) | \e_{\mc{F}-\set{F_i}} = \e^\star_{\mc{F}-\set{F_i}}\right]
  \end{align*}
  Here, the conditional expectation is taken over the random entity
  $\e \sim \Omega$ conditioned on having the same features as
  $\e^\star$ except for $F_i$.
\end{definition}

To see the intuition behind our definition, fix some value $v$, and
let $\e = \e^\star[F_i := v]$.  If $(F_i,v)$ is a counterfactual
cause, then the difference $L(\e^\star) - L(\e)$ is 1; otherwise the
difference is 0.  The \counter-score is simply the expected value of
this difference.

We explain now the importance of the probability space $\Omega$, over
which we take the expectation $\E[\ldots]$.  Suppose a bank is
deciding loan applications for customers in the US.  A customer,
Alice, was denied the loan, and she asks for an explanation.  After
examining millions of prior customers, the clerk notices that one
customer had exactly the same features as Alice, but was born in
Luxembourg, and was granted the loan. Luxembourg is a tiny state in
Europe whose population is financially very savvy, which explains why
that loan was granted while Alice was denied.  Thus,
place-of-origin$=$Luxembourg is a counterfactual cause.  However, it
is a poor explanation, because it is not representative of the
population for Alice.  The definition of \counter-score takes this
into account:
$\counterScore(\texttt{Alice}, \texttt{place-of-origin})$ is very
small, because the vast majority of customers identical to Alice
except for place-of-origin were also denied the loan application.
Thus, our bank will provide Alice with a different explanation, one
having a larger \counter-score.

It is possible for an entity $\e^\star$ to have no counterfactual
cause.  This happens when, changing any single feature to any other
value, the modified entity $\e$ has the same outcome
$L(\e^\star) = L(\e)= 1$.  A pair $(F_i,v)$ is called an {\em actual
  cause} with contingency $(\Gamma, \mathbf{w})$, where $\Gamma$ is a
set of features and $\mathbf{w}$ is a set of values, if $(F_i,v)$ is a
counterfactual cause for
$\e^\star[\Gamma :=
\mathbf{w}]$~\cite{DBLP:journals/corr/abs-1301-2275,DBLP:journals/pvldb/MeliouGMS11}.
In other words, denoting $\e' \defeq \e^\star[\Gamma := \mathbf{w}]$
and $\e'' \defeq \e'[F_i := v]$, the pair $(F_i,v)$ is an actual cause
for $\e^\star$ with contingency $(\Gamma, \mathbf{w})$ iff
$L(\e^\star) = L(\e')=1 \neq L(\e'')=0$.  For an illustration, suppose
Alice is 30 years old and denied the loan.  Setting
$\texttt{NumberOfInquiresLast7Days}=0$ will not change the outcome,
however, if Alice where 5 years older, then setting
$\texttt{NumberOfInquiresLast7Days}=0$ would grant Alice the loan.
Thus, $\texttt{NumberOfInquiresLast7Days}=0$ is an actual cause with
contingency set $(\texttt{Age},35)$.  Chockler and
Halpern~\cite{DBLP:journals/corr/abs-1301-2275} defined the {\em
  responsibility} of an actual cause $(F_i,v)$ with contingency
$(\Gamma,\mathbf{w})$ as $1/(1+|\Gamma|)$; intuitively, smaller
$\Gamma$ results in a larger responsibility, in particular a
counterfactual cause has responsibility 1, because then
$\Gamma=\emptyset$.  

We introduce now the \resp-explanation score.\footnote{In
  \cite{DBLP:journals/mst/BertossiS17} contingency sets for causes at
  the attribute- value level for query answers from databases were
  defined along similar lines.}


\begin{definition}
  Fix an entity $\e^\star$, and a contingency $(\Gamma,\mathbf{w})$
  such that $L(\e^\star) = L(\e')$, where
  $\e' \defeq \e^\star[\Gamma := \mathbf{w}]$.  The \resp-score of a
  feature $F_i$ w.r.t. to the contingency $(\Gamma,\mathbf{w})$ is:
  \begin{align*}
    \respScore(\e^\star,F_i,\Gamma,\mathbf{w}) \ \defeq \ \frac{L(\e') - \E\left[L(\e'') | \e''_{\mc{F}-\set{F_i}} = \e'_{\mc{F}-\set{F_i}}\right]}{1+|\Gamma|}
  \end{align*}
  We define $\respScore(\e^\star, F_i)$ 
  as $\max_{\mathbf{w}} \respScore(\e^\star,F_i, \Gamma,\mathbf{w})$,
  where $\Gamma$ is the smallest set for which this score is $\neq 0$.
\end{definition}


In other words, the \resp-explanation score is defined as follows.  We first set $\Gamma=\emptyset$, in which case
$\respScore(\e^\star,F_i,\Gamma,\emptyset) = \counterScore(\e^\star,F_i)$.  If
this is non-zero, then its value defines the \resp-score.  If it is zero, then
we increase $\Gamma$, until we find a contingency such that
$\respScore(\e^\star,F_i,\Gamma,\mathbf{w}) \neq 0$; this value represents the
\resp-score.

\section{The Shapley-Score}\label{sec:shapley}

Motivated by machine learning applications in the medical domain,
Lundberg and Lee~\cite{DBLP:conf/nips/LundbergL17} have proposed the
Shapley explanation score of a feature $F_i$, in short \shap-score.
This score is not grounded in causality, but instead it is based on
the {\em Shapley value} of cooperative
games~\cite{shapley:book1952}.\footnote{In \cite{DBLP:journals/corr/abs-1904-08679} Shapley scores have been assigned to database tuples to quantify their contribution to query results.}  We review it here
briefly.

Fix an entity $\e^\star$ and a feature $F_i$.  Let $\pi$ be a
permutation on the set of features $\mc{F}$; in other words, $\pi$
fixes a total order on the set of features.  Denote by $\pi^{<F_i}$
the set of features $F_j$ that come before $F_i$ in the order $\pi$;
similarly, $\pi^{\leq F_i}$ denotes $\pi^{<F_i} \cup \set{F_i}$.  The
{\em contribution} of the feature $F_i$ is defined as:
\begin{align*}
c(\e^\star, F_i, \pi) \! \defeq\!\! \E\left[L(\e) | \e_{\pi^{\leq F_i}} = \e^\star_{\pi^{\leq F_i}}\right]\!-\E\left[L(\e) | \e_{\pi^{< F_i}} = \e^\star_{\pi^{< F_i}}\right]
\end{align*}

\begin{definition}
  Fix an entity $\e^\star$.  The \shap-score of a feature $F_i$ is the
  average contribution of $F_i$ over all permutations $\pi$, in other words:
  \begin{align*}
    \shapley(\e^\star, F_i) \defeq & \frac{1}{n!} \sum_\pi c(\e^\star, F_i, \pi)
  \end{align*}
  where $n = |\mc{F}|$ is the number of features.
\end{definition}

The intuition is as follows.  Extend the classifier $L$ to entities
with missing features, as follows.  If $\e^\star_S$ is an entity with
features $S \subseteq \mc{F}$, then define $L'$ to be the expected
value over the missing features:
$L'(\e^\star_S) \defeq \E[L(\e) | \e_S = \e^\star_S]$.  In particular,
when {\em all} features are missing, then $L'(\emptyset) = \E[L(\e)]$,
and when all features are present then $L'(\e^\star) = L(\e^\star)$.
Consider the following process: we present the features of $\e^\star$
to the classifier $L'$ one by one, in some order $\pi$.  The output of
$L'$ changes, step by step, from $\E[L(\e)]$ to $L(\e^\star)$.  The
contribution $c(\e^\star, F_i, \pi)$ of the feature $F_i$ represents
the amount of change observed when we introduce $F_i$.  The
\shap-score simply averages this contribution over all permutations
$\pi$.

The appeal of the \shap-score is that it splits the difference
$L(\e^\star) - \E[L(\e)]$ among the $n$ features $F_1, \ldots, F_n$, which means:
\begin{align}
\sum_i \shapley(\e^\star,F_i) \ = & \ L(\e^\star) - \E[L(\e)] \label{eq:shap:sum}
\end{align}
This follows immediately from the fact that, for any fixed permutation
$\pi$, $\sum_i c(\e^\star, F_i, \pi) = L(\e^\star) - \E[L(\e)]$.
Several good properties of the \shap-score are discussed
in~\cite{DBLP:journals/corr/abs-1905-04610}.  However, unlike the
\resp-score, there is no causal semantics associated to the
\shap-score.  In particular, it is possible for the \shap-score of
some feature to be $<0$, a fact that we observed in our experiments.

We now explain the connection between the \shap-score and the
causality-based scores introduced in Sec.~\ref{sec:da}.  Fix a set
$S \subseteq \mc{F}-\set{F_i}$, and define the contribution of $F_i$
w.r.t. $S$ as:

\begin{align*}
  c'(\e^\star, F_i, S)  \defeq & \ \E\left[L(\e) | \e_{S \cup \set{F_i}}=\e^\star_{S \cup \set{F_i}}\right] -\E\left[L(\e) | \e_{S}=\e^\star_{S}\right]
\end{align*}
For a number $0 \leq \ell < n$, denote by
${\mc{F}-\set{F_i} \choose \ell}$ the subsets of size $\ell$ of
$\mc{F}-\set{F_i}$.  We define the \shap-score at level $\ell$ as
\begin{align}
  \shapley(\e^\star, F_i, \ell) \defeq \ & \ \frac{\ell!(n-\ell-1)!}{n!} \sum_{S \in {\mc{F}-\set{F_i} \choose \ell}} c'(\e^\star, F_i, S)
\label{eq:shapley:2}
\end{align}

It is immediate to check that the \shap-score is the sum of all $n$
levels,
$\shapley(\e^\star, F_i) = \sum_{\ell = 0,n-1} \shapley(\e^\star, F_i,
\ell)$.  Using this property, we prove the following connection
between the \shap-score and the causality-based scores:

\begin{proposition}
$\!\shapley(\e^\star\!, F_i, n-1)\!=\!\frac{1}{n}\counterScore(\e^\star\!, F_i)$.
\end{proposition}

The proof follows immediately from the definitions.  This ``connection'' reveals
more about how different the \resp- and \shap-scores are.  Level $n-1$ is only
one of the many levels of the \shap-score, while the \resp-score agrees with the
\counter-score only when the contingency is empty. Thus, while the \resp- and
\shap-scores are somewhat correlated, they are derived from different principles
(causality vs. Shapley value) and have different mathematical definitions.

\section{Probability Spaces and Algorithms} \label{sec:prob:space}

A key difficulty in computing the \resp- and \shap-scores consists in
defining the probability space $\Omega$.  In practice we do not have
access to the population defining $\Omega$, but only to a sample, for
example the training data, or the test data $T$.  The \resp- and
\shap-scores require the computation of many conditional expectations,
and it is not possible to properly estimate them from a sample $T$.  In this
paper we propose two simple approaches to defining $\Omega$, and study
how the \resp- and \shap-scores can be computed in each case.

In this section we assume to have a dataset $T$ with $N$ tuples and
$n+1$ attributes, $F_1, \ldots, F_n, C$, where each row represents an
entity $\e$ and $C$ represents a count, i.e. the number of times $\e$
occurs in the sample.  We denote by $D_i \defeq \Pi_{F_i}(T)$, the
domain of the feature $F_i$.  Our goal is to define a probability
space $\Omega$ over $D_1 \times \cdots \times D_n$ that is a
generative model for $T$, over which we compute the \resp- and
\shap-scores.

\subsection{The Product Space}

Let $p(F_i=x)$ be the observed marginal probability of the value
$F_i=x$ in the data $T$.  The {\em product space}, $\Omega_P$, is
defined by choosing all feature values independently:
\begin{align*}
  p(\langle x_1, \ldots, x_n\rangle) \defeq & \prod_i p(F_i = x_i)
\end{align*}

The advantage of the product space is that it covers the entire
domain $D_1 \times \cdots \times D_n$, and it preserves the marginal
probabilities of each feature.  On the negative side, $\Omega_P$ does not capture any correlations.

\begin{figure}
  \begin{center}
  \begin{minipage}{5.0\linewidth}
    \begin{tabbing}
      \texttt{//} \= \texttt{create views from all domains:} \\
      \texttt{create view D$_1$ as  select F$_1$, sum(C) from T group by F$_1$}\\
      \texttt{create view D$_2$ as  select F$_2$, sum(C) from T group by F$_2$}\\
      \texttt{$\ldots$} 
      \\
      \texttt{// Compute the RESP-score of F3 with $\Gamma= \set{F1,F4}$:}  \\
      \texttt{with temp as} \\
      \>\texttt{(select y.F$_1$, z.F$_4$, L(y.F$_1$, F$_2^\star$, x.F$_3$, z.F$_4$, F$_5^\star$, $\ldots$)*sum(x.C)/M as S}\\
      \>\texttt{from D$_1$ y, D$_4$ z, D$_3$ x} \\
      \>\texttt{group by y.F$_1$, z.F$_4$} \\
      \>\texttt{having L(y.F$_1$, F$_2^\star$, F$_3^\star$, z.F$_4$, F$_5^\star$, F$_6^\star$, $\ldots$)$=$1)}\\
      \texttt{select max(1-S)/3 from temp}
    \end{tabbing}
  \end{minipage}
    \end{center}
  \vspace{-.8em}
\caption{Illustration of how to compute the \resp-score over the product domain
  $\Omega_P$ using SQL queries.  $T$ is the input data,
  $M := \sum_{\e \in T} \e.C$ is a constant, and the classifier $L(\cdots)$ is a
  User Defined Function.  The input entity $\e^\star$ is given by constants
  $F1^\star, F2^\star, \ldots$, and we assume $L(F1^\star, F2^\star, \ldots)=1$.
  The figure shows how to compute the \resp-score for $F_3$, with contingency
  $F_1, F_4$.  The division by 3 represents the division by $1+|\Gamma|$.}
  \label{fig:algorithm:resp} \vspace*{-1em}
\end{figure}

{\bf Score Computation Algorithm} We first show how to compute
the \resp-explanation.  For that we need to compute the \resp-score of
each feature, then return the feature with highest score.  For a fixed
feature $F_i$, we start by computing the \counter-score, by applying
Def.~\ref{def:counter} directly:
\begin{align*}
  \E\left[L(\e) | \e_{\mc{F}-\set{F_i}} = \e^\star_{\mc{F}-\set{F_i}}\right] = & \sum_{x \in D_i} L(\e^\star[F_i := x]) p(F_i = x)
\end{align*}
Thus, we need to perform a single scan over the domain $D_i = \Pi_{F_i}(T)$ and
at most $|D_i|$ calls to the oracle $L$. In practice, most often
$\counterScore \neq 0$, and then this is the \resp-score of the feature.
Otherwise, we consider contingency sets of size 1, 2, $\ldots$ until we find a
non-zero score.

 We can compute the \resp-score for $F_i$ using SQL queries. First,
  we create a view of the domain $D_j$ for each feature $F_j \in
  \mathcal{F}$. Then, given a set $\Gamma$ of features, we can compute the
  $\max_{\bm w}\resp(\e, F_i, \Gamma, \bm w)$ in one query that computes an
  aggregate over the Cartesian product of the domains of $F_i$ and each feature in
  $\Gamma$.  Fig.~\ref{fig:algorithm:resp} sketches the queries for the case
  when the feature $F_i$ is $F_3$, and the contingency set is $\{ F_1, F_4\}$.

We now turn to the task of computing the \shap-score.  Unfortunately,
the \shap-score is intractable in this model. We show that it is
\#P-hard, even if classifier $L$ is a white-box:

\begin{theorem}
  Suppose all features are binary, $F_i \in \set{0,1}$.  Then the
  following problem is \#P-hard: {\em ``Given a classifier $L$
    specified as a monotone 2CNF function with variables
    $F_1, \ldots, F_n$, and an entity $\e^\star \in \set{0,1}^n$,
    compute the \shap-scores of all its features''.}
\end{theorem}

\begin{proof}
  The following problem is \#P-complete~\cite{DBLP:journals/siamcomp/Valiant79}:
  {\em ``Given a monotone 2CNF Boolean formula
    $L = \bigwedge_{(i,j) \in E} (F_i \vee F_j)$ (where
    $E \subseteq [n] \times [n]$), compute the number of truth assignments
    $\#L$''}.  It follows immediately that the following problem is also
  \#P-hard: compute the probability $p(L) \defeq \#L/2^n$, when each variable
  $F_i$ is set independently to $1$ with probability 1/2.  We describe
  polynomial-time Turing reduction from the \shap-scores computation problem to
  the probability computation problem.  Let $L$ be any monotone 2CNF for which
  we want to compute $p(L)$.  Consider the following dataset $T$ with two
  entities: $\langle 0, 0, \ldots, 0 \rangle$ and
  $\langle 1, 1, \ldots, 1 \rangle$, where the count is $C=1$ for both entities.
  Then each marginal probability is $p(F_i = 0) = p(F_i = 1) = 1/2$, and the
  product space is precisely that in which we want to compute $p(L)$, i.e. all
  Boolean variables $F_i$ are set independently to true with probability
  $1/2$. Consider the input entity $\e^\star = \langle 1, 1, \ldots, 1 \rangle$
  (i.e. all variables are set to 1), thus $L(\e^\star) = 1$ (because the formula
  is monotone).  With $n$ calls to the oracle for $\shapley$, we can obtain
  $p(L)$ as follows:
  $p(L) = \E[L(\e)] = L(\e^\star) - \sum_i \shapley(\e^\star, F_i) = 1 - \sum_i
  \shapley(\e^\star, F_i)$, by \eqref{eq:shap:sum}.
\end{proof}

Thus, we cannot hope to compute the \shap-score over the product
space $\Omega_P$. Instead, we consider the empirical distribution.

\subsection{The Empirical Distribution}\label{sec:emp:space}

The {\em empirical distribution} defined by $T$ is simply $T$ itself.
In other words, the outcomes with non-zero probability are precisely
the tuples in $T$, and the probabilities are given by their
frequencies in $T$.  We denote $\Omega_E$ the empirical distribution
defined by the set $T$.  An advantage of the empirical distribution
is that it captures not only the marginal probabilities, but also the
correlations between features.  A disadvantage of the empirical
distribution is that it associates a zero probability to every unseen
entity.


{\bf Score Computation Algorithm} We start by describing how to compute the
\shap-score over the probability space $\Omega_E$. We apply
Eq.\eqref{eq:shapley:2}, and compute the conditional expectation
$E[L(\e)|\e_S = \e^\star_S]$ for each set of features
$S \subseteq \set{F_1, \ldots, F_n}$.  Each conditional expectation requires a
complete pass over the data $T$, which becomes impractical when $n$ is larger
than 20 or so.  We propose to use an optimization borrowed from the Apriori
algorithm~\cite{DBLP:conf/vldb/AgrawalS94}.  As the set $S$ increases, the set
of entities $\e \in T$ that satisfy $\e_S = \e^\star_S$ decreases, until it
becomes a singleton $\set{\e}$.  At that point, for every superset $S'$ the
conditional expected value is the same as $L(\e)$, and hence we can stop
increasing the set $S$.  In fact, we can stop even earlier: when all entities in
the set $\setof{\e}{\e_S = \e^\star_S}$ have the same outcome $L(\e)$ (either 0
or 1).  While the worst case runtime of the algorithm is still exponential in
$n$, in practice this optimization is quite effective and we were able to
compute the \shap-score for up to $n=30$ attributes.  We present the algorithm
to compute the \shap-scores in Appendix~\ref{appendix:shap:algo}.


While it is possible to compute the \resp-score in this model,
unfortunately this score is meaningless: for most entities $\e^\star$
the \resp-score will require a very large contingency set, leading to
meaningless explanations.  To see the intuition behind this, fix the
set $T$ and consider a new entity $\e^\star$, not necessarily in $T$.
For example, $T$ is our training set, while $\e^\star$ is a new,
random customer, not present in $T$.  In order to have a
non-zero \resp-score with empty contingency set, we need to find some
entity $\e \in T$ that agrees in all features, except one, with
$\e^\star$. This is very unlikely given a random choice for
$\e^\star$.  Extending this simple observation we prove:

\begin{theorem} \label{th:empirical:sparse} Let $L$ be any classifier,
  and let $T \subseteq D_1 \times \cdots \times D_n$ be a set of size
  $N$ that defines the empirical probability space $\Omega_E$.  Fix an
  integer $c \geq 0$.  Then, for a randomly chosen entity $\e^\star$,
  the probability (over the random choices of $\e^\star$) that
  $\respScore(\e^\star, F_j, \Gamma) \neq 0$, for some feature $F_j$
  and some contingency set $\Gamma$ of size $\leq c$, is
  $\leq N (1+\sum_j |D_j|)^{c+1}| / \prod_j |D_j|$.
\end{theorem}

\begin{proof}
  If $\respScore(\e^\star, F_j, \Gamma) \neq 0$ for some feature $F_j$
  and some contingency $\Gamma$, then there must exists two entities
  $\e', \e'' \in T$ such that (a) $\e^\star$ and $\e'$ agree on all
  features except $\Gamma$, (b) $\e'$ and $\e''$ agree on all features
  except $F_j$, and (c) $L(\e^\star)=L(\e') \neq L(\e'')$.  In
  particular, (a) and (b) imply that $\e^\star$ satisfies the
  following property:
  \begin{align}
    \exists \e'' \in T: \mbox{ $\e^\star$ and $\e''$ agree on all but $c+1$ features}
 \label{eq:prob:counter}
  \end{align}
  We claim that its probability is
  $\leq N (1+\sum_j |D_j|)^{c+1}| / \prod_j |D_j|$, which proves the
  theorem.  To prove the claim, start by fixing an entity
  $\e'' \in T$.  Consider a set of features $(F_j)_{j \in J}$, for
  some set $J \subseteq [n]$: the probability that a randomly chosen
  entity $\e^\star$ agrees with $\e''$ on all features except those in
  $J$ is
  $1 / \prod_{k \in [n]-J} |D_k| = \prod_{j \in J} |D_j| / \prod_{k
    \in [n]} |D_k|$.  By the union bound, the probability that a
  randomly chosen $\e^\star$ agrees with $\e''$ on {\em any} set of
  features $J$, with size $|J| \leq c$, is:
  $\leq \sum_{J: |J| \leq c} \prod_{j \in J} |D_j| / \prod_{k \in [n]}
  |D_k| \leq (1+\sum_j |D_j|)^{c+1}| / \prod_j |D_j|$.  Finally, the
  claim follows from the union bound applied to $\e'' \in T$.
\end{proof}

In essence, the theorem says we are very unlikely to find good
\resp-explanations using the empirical space.  For example, assume $n=30$
features, each with a domain of size $|D_i|=10$, and a test data $T$ with
$N=10000$ entities.  We are interested in the probability that a randomly chosen
$\e^\star$ has non-zero \resp-score with a contingency set of size $\leq c$.  By
the theorem, this probability is
$\leq 10000 \cdot (1+ 300)^{c+1}/ 10^{30} \approx 300^{c+1} \cdot 10^{-26}$.
When $c=8$, this quantity is $2\cdot 10^{-5}$.  It is {\em very} unlikely that a
randomly chosen entity will have a explanation with a contingency of size 8 or
smaller.  On the other hand, explanations with contingency sets 8 or larger
become meaningless.  In other words, the empirical distribution is not practical
for computing the \resp-score.

\vspace*{-1em}
\subsection{Discussion}

In this paper we will use the marginal probability space $\Omega_P$
when computing the \resp-score, and will use the empirical probability
space $\Omega_E$ when computing the \shap-score.  We are forced with
this choice by our two results above.  One of the main take-aways of
this paper is that future work on explanation needs to explore more
sophisticated choices for the probability space $\Omega$.


\begin{table*}[t]
  \begin{center}
      \begin{scriptsize}
        \begin{tabular}{|l|l||c|c|c|c|c|c|}\hline
           \multirow{2}{*}{Feature Name}         &  \multirow{2}{*}{Subscale Name }                                   &  \multirow{2}{*}{Entity}  & Feature                 & Subscale            & Subscale         & Subscale                                 & Global              \\ 
                                     &                                     &  &  Score                &  Risk           &  Weight        &  Score                                &  Risk              \\ \hline\hline
  {\color{DarkGreen}\bf ExternalRiskEstimate}     & {\color{DarkGreen}\bf ExternalRiskEstimate}         & 61            & {\color{DarkGreen}\bf 2.9896} & 0.8262                  & 1.566                  & {\color{DarkGreen}\bf 1.2934}                  & \multirow{22}{*}{0.6146} \\ \cline{1-7}
  MSinceOldestTradeOpen                          & \multirow{3}{*}{TradeOpenTime}                     & 198           & 0.2453                       & \multirow{3}{*}{0.4690} & \multirow{3}{*}{2.527} & \multirow{3}{*}{1.1842}                       &                          \\
  MSinceMostRecentTradeOpen                      &                                                    & 14            & 0.0311                       &                         &                        &                                               &                          \\
 AverageMInFile                                  &                                                    & 96            & 0.2960                       &                         &                        &                                               &                          \\ \cline{1-7}
   NumSatisfactoryTrades                         & NumSatisfactoryTrades                              & 25            & 0.0001                       & 0.4513                  & 2.156                  & 0.9729                                        &                          \\ \cline{1-7}
  NumTrades60Ever2DerogPubRec                    & \multirow{4}{*}{TradeFrequency}                    & 0             & 0.0003                       & \multirow{4}{*}{0.4425} & \multirow{4}{*}{0.359} & \multirow{4}{*}{0.1588 }                      &                          \\
  NumTrades90Ever2DerogPubRec                    &                                                    & 0             & 0.1515                       &                         &                        &                                               &                          \\
  NumTotalTrades                                 &                                                    & 27            & 0.2653                       &                         &                        &                                               &                          \\
 NumTradesOpeninLast12M                          &                                                    & 0             & 0.0000                       &                         &                        &                                               &                          \\ \cline{1-7}
 {\color{DarkGreen}\bf PercentTradesNeverDelq}    & \multirow{4}{*}{{\color{DarkGreen}\bf Delinquency}} & 89            & {\color{DarkGreen}\bf 0.5686} & \multirow{4}{*}{0.6847} & \multirow{4}{*}{2.545} & \multirow{4}{*}{{\color{DarkGreen}\bf 1.7425}} &                          \\
  MSinceMostRecentDelq                           &                                                    & 1             & 0.4015                       &                         &                        &                                               &                          \\
 {\color{DarkGreen}\bf MaxDelq2PublicRecLast12M } &                                                    & 4             & {\color{DarkGreen}\bf 1.0046} &                         &                        &                                               &                          \\
 MaxDelqEver                                     &                                                    & 6             & 0.0000                       &                         &                        &                                               &                          \\ \cline{1-7}
 PercentInstallTrades                            & \multirow{3}{*}{ Installment }                     & 11            & 0.0009                       & \multirow{3}{*}{0.5273} & \multirow{3}{*}{0.913} & \multirow{3}{*}{0.4817}                       &                          \\
 NetFractionInstallBurden                        &                                                    & 75            & 0.3706                       &                         &                        &                                               &                          \\
NumInstallTradesWBalance                         &                                                    & 2             & 1.4898                       &                         &                        &                                               &                          \\ \cline{1-7}
MSinceMostRecentInqexcl7days                     & \multirow{3}{*}{Inquiry}                           & 11            & 0.8318                       & \multirow{3}{*}{0.3172} & \multirow{3}{*}{3.004} & \multirow{3}{*}{0.9529}                       &                          \\
NumInqLast6M                                     &                                                    & 0             & 0.0002                       &                         &                        &                                               &                          \\
NumInqLast6Mexcl7days                            &                                                    & 0             & 0.0000                       &                         &                        &                                               &                          \\ \cline{1-7}
NetFractionRevolvingBurden                       & \multirow{2}{*}{RevolvingBalance}                  & 67            & 1.3938                       & \multirow{2}{*}{0.6500} & \multirow{2}{*}{1.924} & \multirow{2}{*}{1.2505}                       &                          \\
NumRevolvingTradesWBalance                       &                                                    & 7             & 0.1176                       &                         &                        &                                               &                          \\ \cline{1-7}
NumBank2NatlTradesWHighUtilization               & Utilization                                        & 2             & 0.8562                       & 0.6490                  & 0.987                  & 0.6406                                        &                          \\ \cline{1-7}
 PercentTradesWBalance                           & TradeWBalance                                      & 75            & 0.4528                       & 0.6113                  & 0.296                  & 0.1808                                        &                          \\ \hline
\end{tabular}
\end{scriptsize}\end{center}
\caption{(left) Features and subscales of the classifier for the FICO dataset.  (right) Entity $\bm e$, and the scores, weights, and risk for both the features and subscales that are computed by the classifier for $\bm e$.}\label{tbl:fico-model}
\end{table*}

\section{Experimental Evaluation}
\label{sec:rudin}

We empirically evaluate the black-box \resp\ and \shap-explanations on a FICO
dataset, and also how they compare to a well-known white-box explanation for
FICO data.  In Appendix~\ref{appendix:kernelshap}, we also compare our scores with an
approximation of the \shap-score from~\cite{DBLP:conf/nips/LundbergL17}.



We further consider a Kaggle dataset used to detect credit card fraud, for which
the explanation scores require a bucketization of continuous features. We
evaluate how robust the scores are to these transformations. Due to space
constraints, we only preset a summary of the results for this dataset. The
details on the dataset and experiments are presented in Appendix~\ref{appendix:creditcard}.

\smallskip

\noindent{\bf Evaluation Setup.}
All experiments are performed in Python 3.7 using the Pandas library, and on an
Intel i7-4770 3.40GHz/ 64bit/32GB with Linux 3.13.0.

For the \resp-score, we restrict the size of the contingency sets to at most 1. If
an entity does not have an non-zero \resp-score for contingency sets of size 1,
we return no explanation.

\smallskip

\noindent{\bf FICO Dataset.}
We consider the dataset from the public FICO
challenge~\cite{DBLP:journals/corr/abs-1811-12615}. The objective of the
challenge is to provide explanations for credit risk assessments.

The dataset consists of 23 continuous features and 10,459 entities. The features
are shown in the left column of Table~\ref{tbl:fico-model}. The dependent
variable {\it RiskPerformance} encodes whether the applicant will make all
payments within 90 days of being due ({\em good}, 0), or will make a payment
over 90 days after due ({\em bad}, 1). We remove 588 entities from the dataset,
for which all values are missing (indicated by the value -9). This is because we
cannot provide explanations for entities for which we have no information.

We separate the input dataset into training and test data. The test dataset is a
random sample of 1,975 entities (20\% of the original dataset). We use the
training dataset to learn the prediction model, and the test dataset to evaluate
the model and to explain the predictions of the model.

As noted by the FICO community, several input features are {\em monotonically
  increasing}, i.e., the probability of the outcome being bad increases with the
feature value.

\vspace{-1em}
\subsection*{Experiments with FICO Dataset}

We compare the black-box \resp- and \shap-explanations with a white-box
explanation described in\ignore{by Chen et al.}~\cite{DBLP:journals/corr/abs-1811-12615}
specifically for the FICO dataset.  We denote this white-box explanation as
\fico-explanation.  In order to explain it, we need to do a
rather detailed review of the model
used~\cite{DBLP:journals/corr/abs-1811-12615}.


\noindent{\bf Classification Model.}
Our goal was to use the exact model and explanation score
in~\cite{DBLP:journals/corr/abs-1811-12615}, but, unfortunately, the paper does
not provide sufficient information to replicate the model, and the online
demonstration 
of the model does not seem to follow the description in the paper. Hence, we
re-implemented the model using our best understanding of the description in the
paper. Our implementation makes one design choice that differs from the original
model: the input features are bucketized features into disjoint ranges, as
opposed to overlapping ranges. This ensures that each entity has at most one
explanation per feature, whereas the original model can provide several
explanations for a single feature if the feature value falls into several
buckets (c.f., Table 1 in \cite{DBLP:journals/corr/abs-1811-12615}). We next
describe the model as we implemented it.

The classifier is a two-layer neural network, where each layer is defined by
logistic regression models. A logistic regression model $LR_{\bm \theta}(\bm x)$
with features $\bm x = (x_1, \ldots, x_n)$ is defined by $n+1$ weights
$\bm \theta = (\theta_0, \theta_1, \ldots, \theta_n)$, where $\theta_0$ is the
bias term of the model. For a given entity with features $\bm x$, the model
computes $p = \mathtt{sigmoid}(\theta_0 + \sum_{i \in [n]} \theta_i\cdot x_i)$, which returns a value between 0 and 1. We
refer to the outcome of $LR$ as the probability of risk. The model classifies
the entity as a ``bad'' outcome if the probability of risk is above 0.5.

The model requires that the continuous input features are bucketized and one-hot
encoded.  We use exactly the same buckets
as~\cite{DBLP:journals/corr/abs-1811-12615}, and describe them in
Appendix~\ref{appendix:ex:buckets}. One-hot encoding turns the bucketized
feature values into indicator vectors, with one entry for each bucket, which is
$1$ if the value is in this bucket, and $0$ otherwise.  After bucketization and
one-hot encoding, the 23 {\em input features} become 165 {\em binary features},
which are input to the model.


Next, the binary features are categorized into 10 disjoint groups,
each group consisting of all binary features derived from 1-4 input
features, called {\em subscales}, as shown in
Table~\ref{tbl:fico-model}. Features within a group describe similar
or related aspects of an applicant. For example,
\textit{MScienceOldestTradeOpen}, \textit{MScinceMostRecentTradeOpen},
and \textit{AverageMInFile} are categorized into a single
subscale, \textit{TradeOpenTime}, because they are all related to the
number of months that a trade is open.

The first layer of the classifier consists of one logistic regression model for
each subscale. The model for subscale $S$ returns the probability of risk
associated with the features in $S$. 
The second layer is defined by a single logistic regression model, whose inputs
are the subscale risk predictions of the models in the first layer, and whose
output is the risk prediction for the entity. For each feature $F_i$ (subscale
$S$), we store the \emph{feature score} (\emph{subscale score}), which is a dot
product of the binary vector for $F_i$ (subscale risk of $S$) and the
corresponding parameters.

The models are trained in the R library {\it glmnet}, and use
monotonicity constraints to model the monotonicity of the input features
The classification model has an ROC-AUC score of 0.812 and classifies 1020
entities as a high risk entities (label 1). We focus on explaining these
`bad' outcomes.

Table~\ref{tbl:fico-model} depicts, for illustration of the
methodology, the nested structure of the model for one particular
entity $\bm e$ in the dataset.  For each subscale $S$, the LR model is
applied to the binary features of that subscale.  For example, the
first feature In Table~\ref{tbl:fico-model},
\textit{ExternalRiskEstimate}, has value 61, and falls in the first of
the 8 buckets, hence the {\em
  Subscale risk} of this subscale is
$\mathtt{sigmoid}(\theta_0 + \theta_1 \cdot 1 + \theta_2 \cdot 0 +
\cdots \theta_8 \cdot 0)$, which is $0.8262$ in our example.  For
another example, the subscale risk of the \textit{Delinquency}
subscale is $0.6847$.
%
Next, all subscale risks become the input to the logistic regression model at
the second layer: they are multiplied by the weight of the 2nd layer LR, shown
in the column {\em Subscale Weight} in the Table, added up and passed through
the sigmoid function.  This results in the final risk probability of the entity
$\bm e$, which in our example is $0.6146$. Since the risk is $>$0.5, the model
classifies $\e$ as high risk, i.e., $L(\e)=1$.

\begin{figure*}[t]
  \centering
\centerline{\hspace{2cm} \fico-explanation \hfill \resp-explanation \hfill \shap-explanation \hspace{2cm}}
  \includegraphics[width=.32\textwidth]{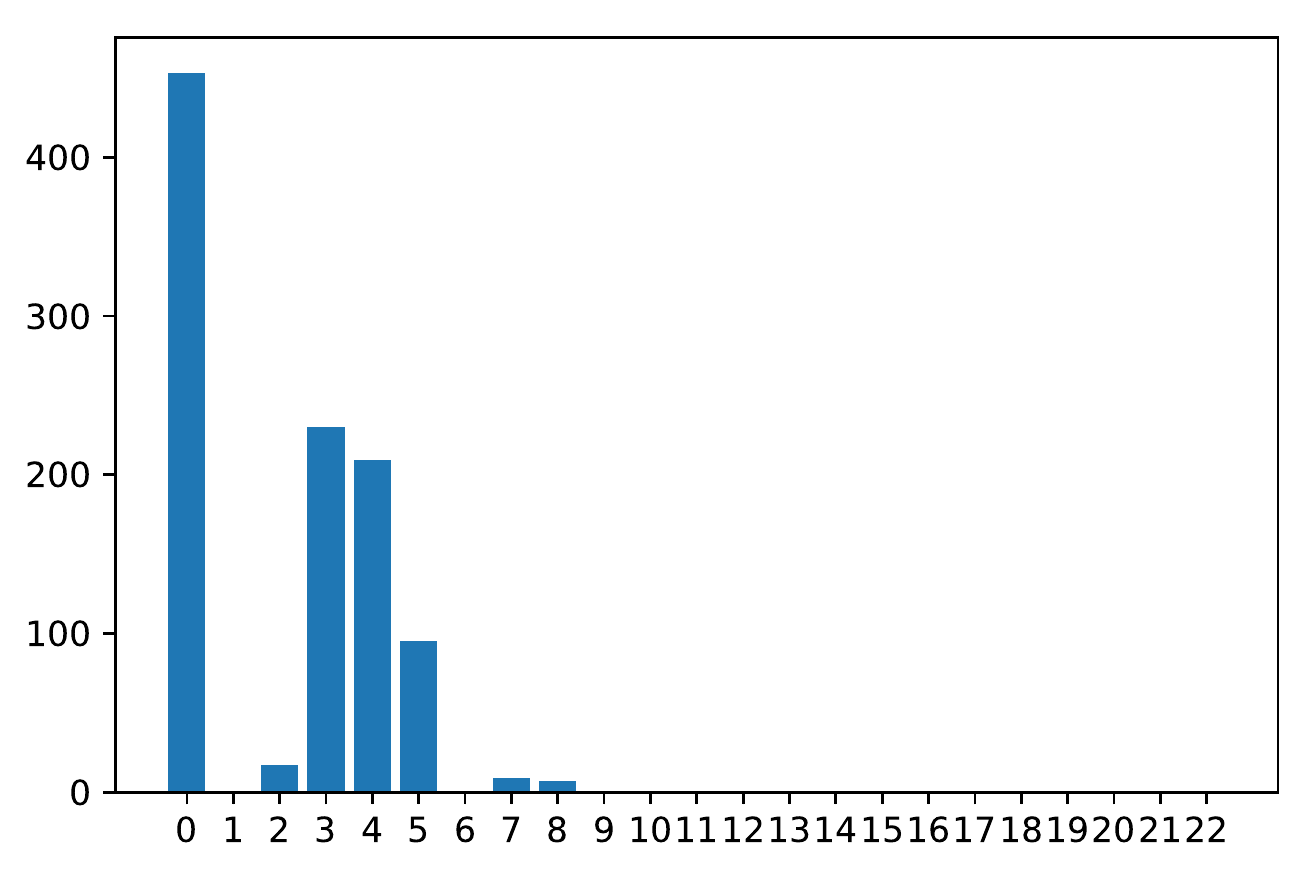}
  \includegraphics[width=.32\textwidth]{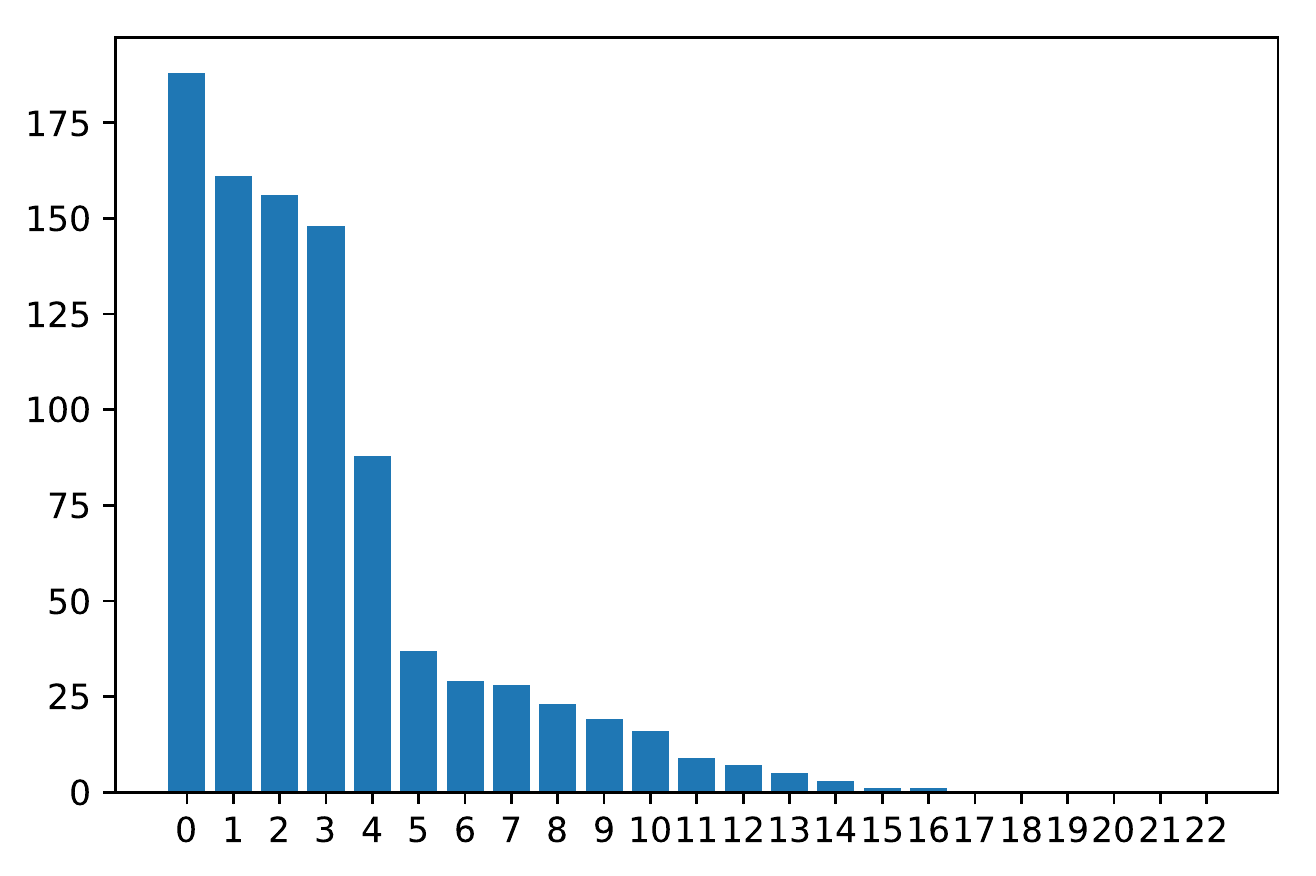}
  \includegraphics[width=.32\textwidth]{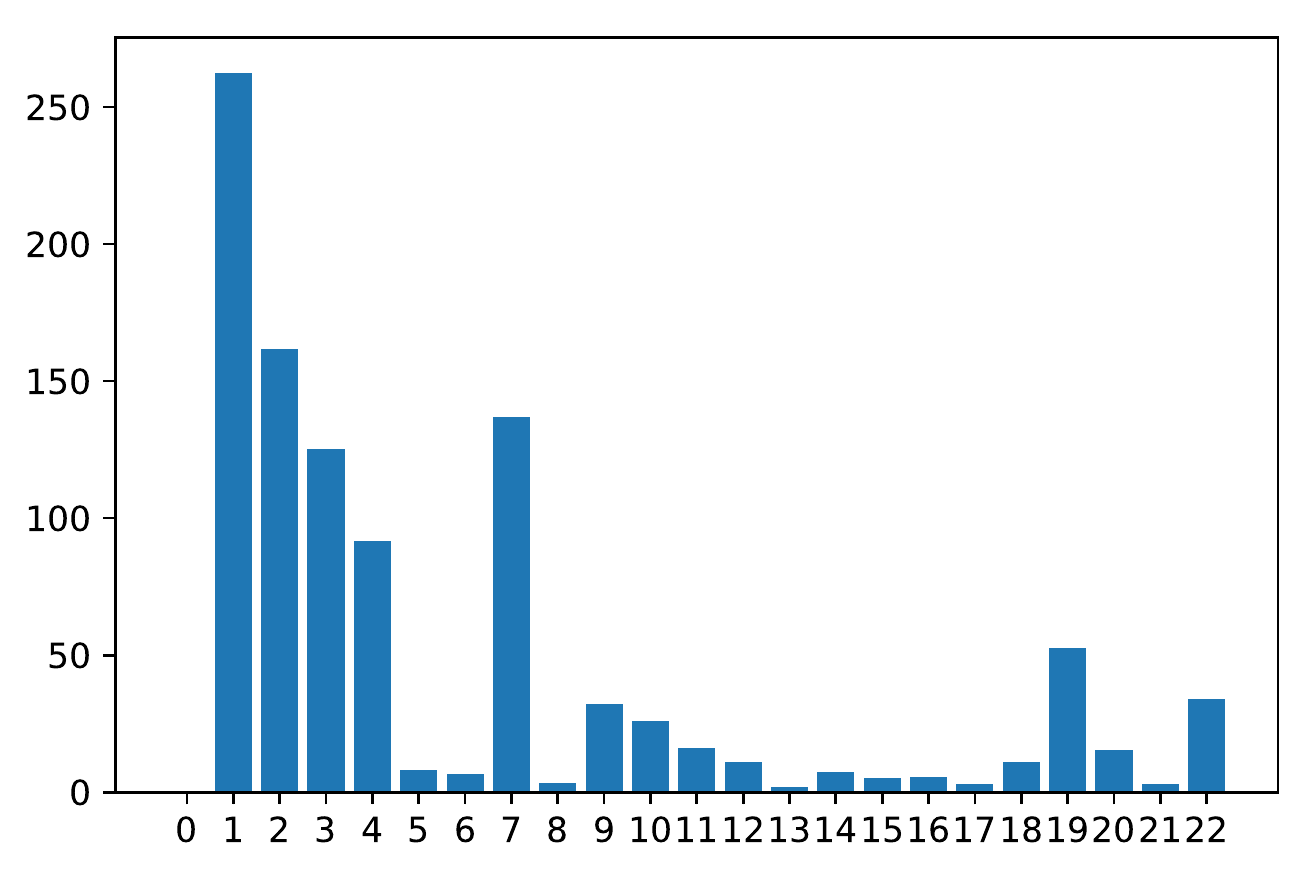}
  \begin{scriptsize}\hspace*{-1em}
  \begin{tabular}{l}
    \emph{Key:} 0 = MSinceMostRecentInqexcl7days; 1 = ExternalRiskEstimate; 2 = NetFractionRevolvingBurden; 3 = PercentTradesNeverDelq; \\
     4 = AverageMInFile; 5 = MaxDelq2PublicRecLast12M; 6 = NumInqLast6M; 7 = MSinceOldestTradeOpen; 8 = NumSatisfactoryTrades;  \\
    9 = NumBank2NatlTradesWHighUtilization; 10 = PercentInstallTrades; 11 = NumTrades60Ever2DerogPubRec;  \\
    12 = NumRevolvingTradesWBalance;
     13 = MaxDelqEver;
    14 = NumTradesOpeninLast12M;  15 = NumTotalTrades; \\
    16 = NumInqLast6Mexcl7days; 17 = MSinceMostRecentTradeOpen; 
    18 = NumTrades90Ever2DerogPubRec; 19 = MSinceMostRecentDelq; \\
    20 = NetFractionInstallBurden; 21 = NumInstallTradesWBalance; 
    22 = PercentTradesWBalance; 
  \end{tabular}
  \end{scriptsize}
  \vspace*{-1em}
  \caption{Distribution of the top ranked features for the fico-, resp-, and
    shap-explanations. Each bar represents for how many of the 1020 entities
    that feature was the top explanation.  To facilitate comparison, all three
    bar charts list the features in the same order, which is the decreasing
    order for the \resp-explanation (middle chart).}
    \vspace*{-1em}
  \label{fig:top1-scores}  
\end{figure*}

\noindent{\bf\fico-explanation.}
We next describe the \fico-explanation
from~\cite{DBLP:journals/corr/abs-1811-12615}. For a given entity $\bm e$, the
\fico-explanation is computed in four steps:
\begin{enumerate}[leftmargin=0.5cm]
\item Run the model on $\bm e$ to compute the subscale and final
  risks, and store the scores for each feature and subscale.
\item Rank the subscales in decreasing order of the subscale scores
  obtained in Step 1, and keep the top two subscales.
\item For each of the top two subscales, rank its input features in
  decreasing order of their feature scores, and keep only the top two 
  (or all features if there are less than two).
\item Concatenate the top two features for each of the top two subscales,
  sorting first by the subscale score, then by the feature score.  This is the
  final \fico-explanation ranking. 
\end{enumerate}

For the example in Table~\ref{tbl:fico-model}, the top subscales and features
are shown in green. The top-2 subscales are \textit{Delinquency} and
\textit{ExternalRiskEstimate}, and the top-2 features in \textit{Delinquency} are
\textit{PercentTradesNeverDelq} and \textit{MaxDelq2PublicRecLast12M}.  Thus,
the final \fico-explanation order is: \textit{PercentTradesNeverDelq},
\textit{MaxDelq2\-PublicRecLast12M}, \textit{ExternalRiskEstimate}.

\begin{figure*}[t]
  \centering
  \includegraphics[width=.25\textwidth]{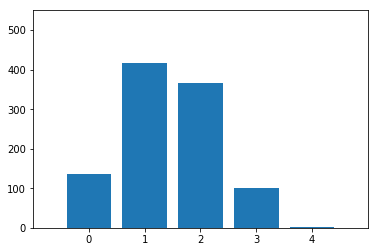}\hspace{3em}
  \includegraphics[width=.25\textwidth]{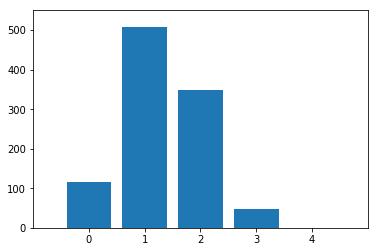}\hspace{3em}
  \includegraphics[width=.25\textwidth]{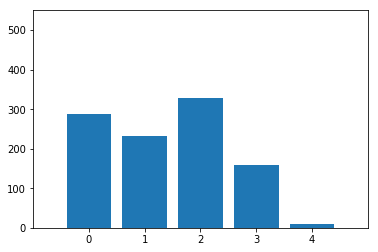}
  \vspace*{-1em}

  \caption{Distribution of the intersection size for top-4
    explanations for the (left) \fico\ and \resp-explanations,
    (middle) \fico\ and \shap-explanations, and (right) \resp\ and
    \shap-explanations.}
  \label{fig:top4-intersection}
\end{figure*}

\noindent{\bf Experimental Result.}
To compare the \fico, \resp, and \shap-explanations, we compare (1)
the distribution of the top explanations for each score, and (2) the
average similarity of the top-4 explanations for each entity between
all pairwise combinations of the scores.  Recall that both \resp- and
\shap-scores are black-box models: although we used the same
classification model as for the \fico-explanation, they only use the
outputs of the model.

We report on the explanations for the 1020 entities for which the
classifier predicts 1, i.e., the `bad' outcome. We use $M = N = 2$ in
Steps 2 and 4 of the \fico-explanations. The \resp-score fails to
provide explanations for 101 entities, due to the restriction on the
size of the contingency set.

\ignore{


}

Fig.~\ref{fig:top1-scores} presents the distribution of the top
explanations returned by the \fico, \resp, and \shap-explanations.

The first observation is a reasonable correlation between the
\fico-explanation (white-box) and the \resp-explanation (black-box).
Their most frequent top feature is the same,
\textit{MSinceMostRecentInqexcl7days}, and there are a few other
features that are popular top explanations for both scores.  

However, there is an obvious difference with \textit{ExternalRiskEstimate},
which is the second most popular explanation for \resp, yet is never the top
\fico-explanation.  To understand it, we looked deeper into the data, and
illustrate our findings with the example that we show in
Table~\ref{tbl:fico-model}.  As we saw, the top-2 subscales are
\textit{Delinquency} and \textit{ExternalRiskEstimate} in this order(!), hence
\textit{ExternalRiskEstimate} will not be the top explanation, instead it will
be preceded by the two features in the \textit{Delinquency} subscale.  Thus,
although the feature-score of \textit{ExternalRiskEstimate} is the highest, it
ranks lower only because of the specific way the \fico-explanation is ranked,
namely it is ranked first by the sub-scale score, and then by the feature score.
This could be adjusted by tweaking the way one computes the ranking.  However,
there is a deeper reason why the \fico-explanation fails to report
\textit{ExternalRiskEstimate} as top explanation: it is because the
\fico-explanation is based only on the {\em current} entity $\e$, and ignores
other entities in the population.  In contrast, the \resp-score is based on
causality and considers the {\em entire population} of
entities. \textit{ExternalRiskEstimate} is a counterfactual cause of our entity,
because by changing its value from $61$ to $81$ the outcome changes from
$L(\e)=1$ to $L(\e)=0$, while neither \textit{PercentTradesNeverDelq} nor
\textit{MaxDelq2PublicRecLast12M} are causal.  The reason why the first is
counterfactual while the others are not lies in the weights $\theta_i$
associated to the buckets of these features.  The values $61$ to $81$ lie in
buckets 1 and 5 of \textit{ExternalRiskEstimate}, respectively, and their
weights vary significantly: $\theta_1 = 2.9896$ and $\theta_5 = 0$. In contrast,
the buckets of \textit{PercentTradesNeverDelq} and
\textit{MaxDelq2PublicRecLast12M} have weights in a small range $[1.5,1.9]$, and
changing their values is insufficient to change the outcome.  The
\fico-explanation looks only at the current bucket, and fails to notice that
other buckets have significantly different values.  In contrast, the \resp-score
considers the entire model because it examines the entire population, checking
for a counterfactual feature.


In general, we observe that the top \fico-explanations are correlated
to the weights in the second layer of the classifier. In fact, the
four most common top explanations come from the three subscales with
the highest weight (their weights are at least 2.5,
c.f. Table~\ref{tbl:fico-model}). Thus, they are more likely to have a
high subscale score, and consequently, to be among the top subscales
in Step 2 of the \fico-explanation.  This also explains why the
\fico-explanation is less diverse than the \resp-explanation: it tends
to choose features from the same three subscales.  We argue that a
good diversity is a desired quality of an explanation score: we want
to be able to give individualized explanations to the customers, and,
assuming all features are relevant to the outcome, we expect a diverse
distribution of the top explanations. Of the three graphs in
Fig.~\ref{fig:top1-scores}, the \resp-explanation has clearly the most
diversity.

Next, we compared the \resp-explanation to the \shap-explanation, and
notice that they are rather distinct.  To understand the source of the
difference, we focused on the fact that \shap-score never returns
\textit{MSinceMostRecentInqexcl7days} as the top-explanation, which is
the most frequent top-explanation for the other two scores.  Recall
that the \shap-score is the sum of $n-1$ levels, see
Eq.~\eqref{eq:shapley:2}, and the weights of the levels is an inverse
binomial term, $\frac{\ell!(n-\ell-1)!}{n!}$; thus, most of the mass
of the score consists of the first levels $\ell = 0, 1, \ldots$ and
the last levels $\ell = n-1, n-2, \ldots$, while the weights
of the middle levels decrease very fast (exponentially).  After
examining the data, we found that, for each value of
\textit{MSinceMostRecentInqexcl7days}, the distribution of 0
and 1 outcomes in the test dataset are fairly even. As a
result, the first layer of the \shap-score ($\ell = 0$) is always
close to zero.  On the other hand, recall from
Sec.~\ref{sec:prob:space} that for the \shap-score we are forced to
use the empirical probability distribution, and, as we argued there,
the contribution of the higher levels is zero.  This explains why
\textit{MSinceMostRecentInqexcl7days} has a very low \shap-score. We
believe that this is an artifact of the empirical distribution
that underlies the computation of the \shap-score. We conjecture
that this phenomenon would not occur if the \shap-score was computed
over more sophisticated probability spaces.

Finally, we compare the top-4 explanations for each pairwise
combination of the three scores. We assume the top-4 explanations are
sets, and ignore their ranking and scores. We compute two statistics on
each entity: (1) the size of the intersection, and (2) the Jaccard
coefficient for set similarity.

Fig.~\ref{fig:top4-intersection} depicts the distribution of the
intersection size of the top-4 explanations. For the \fico- and
\resp-explanations, we observe that 86.8\% of the entities share a
common explanation. The remaining cases are mostly entities for which
the \resp-score does not provide any explanation, due to the
restriction on the contingency set. For the \fico- and
\shap-explanations, 88.5\% of the entities share a common explanation.
For both comparisons, the number of common explanations is usually
less than three. This is not surprising, because the top-4
\fico-explanations must come from the top-2 subscales, whereas the
explanations for the \resp\ and \shap-scores tend to be more diverse,
by allowing for explanations from different subscales.

The average Jaccard similarity coefficient for each of the pairwise score
comparisons are: (1) 0.276 for the \fico\ and \resp, (2) 0.213 for the \fico\
and \shap, and (3) 0.263 for the \resp\ and \shap. The Jaccard coefficients
underline our observation from Fig.~\ref{fig:top1-scores}: On average the
\fico-explanation is more similar to the \resp- than the \shap-explanation.


Overall, there is only a limited overlap of the top-4 explanations for the three
scores. We believe that this is due to differences in the explanation
methodologies, and the underlying probability space for the \resp- and
\shap-scores.

\subsection*{Experiments with Credit Card Fraud Dataset}

We summarize our findings for the evaluation of the \resp- and \shap-scores on
the Kaggle Credit Card Fraud dataset. The details on the evaluation are provided
in Appendix~\ref{appendix:creditcard}.

All features in this dataset are high-precision continuous variables. The \resp-
and \shap-scores require that such features are bucketized into equi-depth
buckets.

We evaluate how sensitive the two scores are to the number of buckets. We
observe that the \resp-score is robust to the bucketization of the features:
varying the number of buckets has a limited effect on the top
\resp-explanations. For the \shap-scores, however, there is a tradeoff: as the
number of buckets increase, the number of entities after conditioning on any
feature decreases. As a result, the \shap-score fails to give an explanation for
large number of buckets.

As for the FICO dataset, the two explanation scores with the same bucketization
have limited overlap: 85\% of the top-4 explanations have an overlap of at least
two features, yet the Jaccard similarity coefficient is 0.22.
\section{Discussion} 

In this paper we introduced a simple notion of explanation, \resp, for
a classification outcome, which is grounded in the notion of
causality.  We have compared \resp\ to the \shap-explanation and to a
white-box explanation for a specific classification problem.  While no
benchmarks exists for
explanations~\cite{DBLP:journals/corr/Doshi-VelezK17,DBLP:journals/cacm/Lipton18},
our empirical evaluation suggests that \resp\ can provide similar or
better quality.  Our initial goal of this project was to evaluate the
\shap-score, but we soon ran into difficulties due to its high
computational complexity. Lundberg et al.~\cite{DBLP:conf/nips/LundbergL17} claim a polynomial-time implementation, but it is restricted to decision trees. 
In contrast, we have shown that its complexity is
\#P-hard for the product space.  Because of these difficulties, we proposed the
\resp-explanation, based on the simple concept of {\em counterfactual
  cause}, and found it quite natural to interpret its results on real
data.  On the other hand, our experiments have limited the
\shap-explanation to the empirical distribution.  Future work is
needed to evaluate \shap\ on richer probability spaces.

\section*{Acknowledgments}

Suciu was partially supported by NSF IIS 1907997,NSF III
1703281, NSF III-1614738, and NSF AITF 1535565. We thank Guy van den Broeck for
insightful discussions on the complexity of \shap, and Berk Ustun for his
feedback and pointing out the connection to recourse.

\bibliographystyle{plain}
\bibliography{deemScores}

\begin{thebibliography}{10}

\bibitem{DBLP:conf/vldb/AgrawalS94}
Rakesh Agrawal and Ramakrishnan Srikant.
\newblock Fast algorithms for mining association rules in large databases.
\newblock In {\em PVLDB}, pages 487--499, 1994.

\bibitem{DBLP:journals/mst/BertossiS17}
Leopoldo~E. Bertossi and Babak Salimi.
\newblock From causes for database queries to repairs and model-based diagnosis
  and back.
\newblock {\em Theory Comput. Syst.}, 61(1):191--232, 2017.

\bibitem{DBLP:journals/corr/abs-1811-12615}
Chaofan Chen, Kangcheng Lin, Cynthia Rudin, Yaron Shaposhnik, Sijia Wang, and
  Tong Wang.
\newblock An interpretable model with globally consistent explanations for
  credit risk.
\newblock {\em CoRR}, abs/1811.12615, 2018.

\bibitem{DBLP:journals/jair/ChocklerH04}
Hana Chockler and Joseph~Y. Halpern.
\newblock Responsibility and blame: {A} structural-model approach.
\newblock {\em J. Artif. Intell. Res.}, 22:93--115, 2004.

\bibitem{datta2016}
Anupam Datta, Shayak Sen, and Yair Zick.
\newblock Algorithmic transparency via quantitative input influence: Theory and
  experiments with learning systems.
\newblock In {\em IEEE SP}, pages 598--617, 2016.

\bibitem{DBLP:journals/corr/Doshi-VelezK17}
Finale Doshi{-}Velez and Been Kim.
\newblock A roadmap for a rigorous science of interpretability.
\newblock {\em CoRR}, abs/1702.08608, 2017.

\bibitem{DBLP:conf/nips/GhorbaniWZK19}
A.~Ghorbani, J.~Wexler, J.~Y. Zou, and B.~Kim.
\newblock Towards automatic concept-based explanations.
\newblock In {\em NeurIPS}, pages 9273--9282, 2019.

\bibitem{DBLP:journals/corr/abs-1907-07165}
Yash Goyal, Uri Shalit, and Been Kim.
\newblock Explaining classifiers with causal concept effect (cace).
\newblock {\em CoRR}, abs/1907.07165, 2019.

\bibitem{DBLP:journals/corr/abs-1301-2275}
Joseph~Y Halpern and Judea Pearl.
\newblock {Causes and explanations: A structural-model approach. Part I:
  Causes}.
\newblock {\em British J. Philosophy of Science}, 56(4):843--887, 2005.

\bibitem{DBLP:conf/aistats/KhannaKGK19}
Rajiv Khanna, Been Kim, Joydeep Ghosh, and Oluwasanmi Koyejo.
\newblock Interpreting black box predictions using fisher kernels.
\newblock In {\em AISTATS}, pages 3382--3390, 2019.

\bibitem{DBLP:journals/cacm/Lipton18}
Zachary~C. Lipton.
\newblock The mythos of model interpretability.
\newblock {\em Commun. {ACM}}, 61(10):36--43, 2018.

\bibitem{DBLP:journals/corr/abs-1904-08679}
Ester Livshits, Leopoldo Bertossi, Benny Kimelfeld, and Moshe Sebag.
\newblock {The Shapley Value of Tuples in Query Answering}.
\newblock In {\em ICDT}, volume 155, pages 20:1--20:19, 2020.

\bibitem{DBLP:journals/corr/abs-1905-04610}
Scott~M Lundberg et~al.
\newblock From local explanations to global understanding with explainable {AI}
  for trees.
\newblock {\em Nature machine intelligence}, 2(1):2522--5839, 2020.

\bibitem{DBLP:conf/nips/LundbergL17}
Scott~M. Lundberg and Su{-}In Lee.
\newblock A unified approach to interpreting model predictions.
\newblock In {\em NIPS}, pages 4765--4774, 2017.

\bibitem{DBLP:journals/pvldb/MeliouGMS11}
Alexandra Meliou, Wolfgang Gatterbauer, Katherine~F. Moore, and Dan Suciu.
\newblock The complexity of causality and responsibility for query answers and
  non-answers.
\newblock {\em {PVLDB}}, 4(1):34--45, 2010.

\bibitem{pearl}
Judeea Pearl.
\newblock {\em Causality: Models, Reasoning and Inference}.
\newblock Cambridge Univ. Press, 2009.
\newblock 2nd ed.

\bibitem{DBLP:conf/kdd/Ribeiro0G16}
Marco~T{\'{u}}lio Ribeiro, Sameer Singh, and Carlos Guestrin.
\newblock "why should {I} trust you?": Explaining the predictions of any
  classifier.
\newblock In {\em {SIGKDD}}, pages 1135--1144, 2016.

\bibitem{DBLP:journals/corr/abs-1811-10154}
Cynthia Rudin.
\newblock Please stop explaining black box models for high stakes decisions.
\newblock {\em CoRR}, abs/1811.10154, 2018.

\bibitem{shapley:book1952}
L.~S. Shapley.
\newblock A value for n-person games.
\newblock In {\em Contributions to the Theory of Games II}, pages 307--317.
  Princeton University Press, 1953.

\bibitem{deeplift}
Avanti Shrikumar, Peyton Greenside, and Anshul Kundaje.
\newblock Learning important features through propagating activation
  differences.
\newblock In {\em ICML}, pages 3145--3153, 2017.

\bibitem{ustun:recourse}
Berk Ustun, Alexander Spangher, and Yang Liu.
\newblock Actionable recourse in linear classification.
\newblock In {\em FAT}, pages 10--19, 2019.

\bibitem{DBLP:journals/siamcomp/Valiant79}
Leslie~G. Valiant.
\newblock The complexity of enumeration and reliability problems.
\newblock {\em {SIAM} J. Comput.}, 8(3):410--421, 1979.

\bibitem{wachter2017counterfactual}
Sandra Wachter, Brent Mittelstadt, and Chris Russell.
\newblock Counterfactual explanations without opening the black box: Automated
  decisions and the gdpr.
\newblock {\em Harv. JL \& Tech.}, 31:841, 2017.

\end{thebibliography}

\appendix

\section{Computing SHAP}
\label{appendix:shap:algo}

Figure~\ref{fig:algorithm:shap} presents the algorithm to compute the
\shap-score for a given entity $\e^\star$. We assume that the score is computed
over a dataset $T$ with features $F_1, \ldots, F_n$.

The scores are computed in two steps. First, we compute the conditional
expectations $E[L[e] | \e_S = \e_S^\star]$ for all subsets
$S \subseteq \{F_1, \ldots, F_n\}$. They are computed by the {\bf expectations}
function, which enumerates all possible subsets $S$ depth-first. This traversal
allows us to apply the optimizations mentioned in Section~\ref{sec:emp:space}:
if $E[L[e] | \e_S = \e_S^\star] = L(\e^\star)$, we avoid computing the
expectation for any set $S' \supseteq S$, because they all have the same
value. Each conditional expectation can be computed with a single SQL query. For
instance, consider the set $S = \{F_1, F_4\}$ and let $\e^\star_{F_1}$ the value
for feature $F_i$ in $\e^\star$, then we can compute
$V := E[L[e] | \e_S = \e_S^\star]$ as follows:
\begin{align*}
  &\texttt{SELECT AVG(L($F_1, \ldots, F_n$)) AS V}\\
  &\texttt{FROM T WHERE $F_1 = \e^\star_{F_1}$ AND $F_4 = \e^\star_{F_4}$;}
\end{align*}

The {\bf shap} function then computes the \shap-scores for each feature. All
scores are computed in one pass over the precomputed conditional expectations
following Eq.\eqref{eq:shapley:2}.  

\begin{figure}[t]
  \begin{center}
  \begin{minipage}[t]{.5\linewidth}\small
    \begin{tabbing}
      \texttt{{\bf exp}}\=\texttt{{\bf ectations}(Entity $\bm e^\star$, Set $S$,  List $F$): }\\
      \>\texttt{compute V$:=E[L(\e)|\e_S = \e^\star_S]$ \hspace{1.5em} //(one pass over $T$)}\\
      \>\texttt{CONDEXP $:=\emptyset$} \\
      \>\texttt{if }\=\texttt{$V \neq L(\e^\star)$ then:} \\
      \>\>\texttt{for}\=\texttt{all features $f \in F$ do:} \\
      \>\>\>\texttt{CONDEXP $\cup$= {\bf expectations}($ \e^\star, S \cup \set{f}, F_{> f}$)}\\
      \>\texttt{return CONDEXP $\cup$ (S, V);}\\
      \\
      \texttt{{\bf sha}}\=\texttt{{\bf p}(Entity $\bm e^\star$): }\\
      \>\texttt{CONDEXP = \texttt{\bf expectations}($\bm e^\star$, $\emptyset$,
        [$F_1, \ldots, F_n$]);}\\
      \>\texttt{shap = [0 | $f \in \{F_1, \ldots, F_n$\}];}\\
      \>\texttt{for}\=\texttt{all pairs (S,V)$ \in {}$CONDEXP do:} \\
      \>\>\texttt{for}\=\texttt{all features $f \in \{F_1, \ldots, F_n\}$ s.t. $f \notin S$ do:} \\
      \>\>\>\texttt{if }\=\texttt{$\mathtt{S \cup \{f\}} \in {}$CONDEXP then:} \\
      \>\>\>\>\texttt{V' = CONDEXP[$S \cup \{f\}$]} \\
      \>\>\>\>\texttt{shap[f] += $\frac{|S|! (n-|S|-1)!}{n!}$ (V' - V)} \\
      \>\texttt{return shap;}
    \end{tabbing}
  \end{minipage}
\end{center} \vspace{-1em}
\caption{Computation of the \shap-score for a given entity $\e^\star$ over
    dataset $T$ with features $F_1, \ldots, F_n$.  The algorithm first computes all
    required conditional expectations $E[L[e] | \e_S = \e_S^\star]$ for sets $S \subseteq
    \{F_1, \ldots, F_n\}$, and then reuses their computation to compute the score of
    each feature.  We use $F_{> f}$ to denote the list of features that occur after
    $f$ in $F$.}
  \label{fig:algorithm:shap}
\end{figure}

\section{Bucketization for FICO dataset}
\label{appendix:ex:buckets}

We illustrate the bucketization of \textit{ExternalRiskEstimate}.  Its buckets
are $[0,63]$, $[64,70]$, $[71,75]$, $[76,80]$, $[81,\infty]$, $\set{-7}$,
$\set{-8}$, $\set{-9}$, where the last three buckets correspond to special
generic values that indicate missing, outdated, or inapplicable records.  For
the entity in Table~\ref{tbl:fico-model}, $\textit{ExternalRiskEstimate}=61$,
after one-hot encoding, this value is represented as the vector
$[1,0,0,0,0,0,0,0]$, because 61 falls into the first bucket.

\section{Comparison with KernelSHAP}
\label{appendix:kernelshap}

\begin{figure}[t]\centering
  \includegraphics[width=.5\columnwidth]{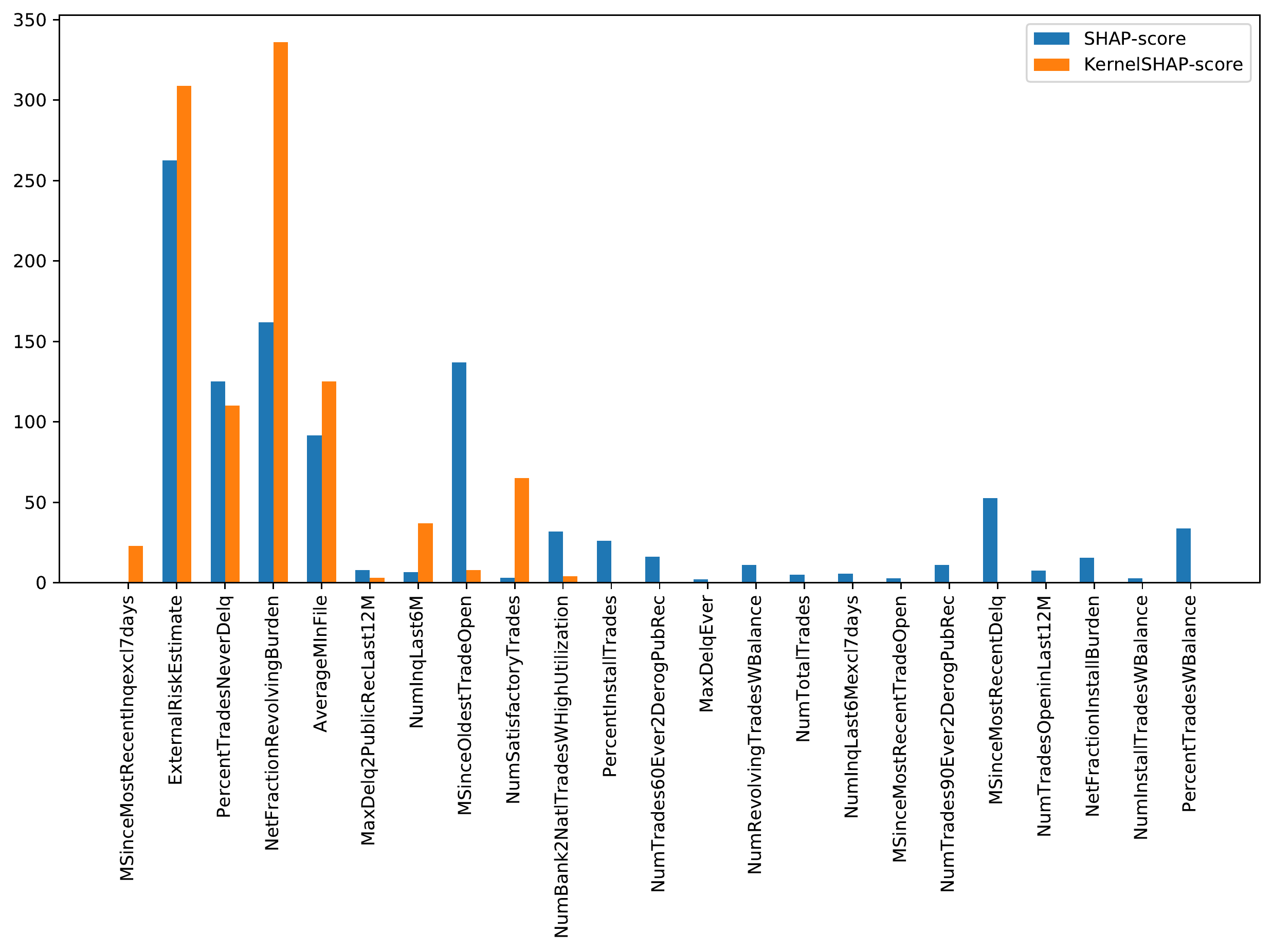}
  
  \caption{Comparison of the distribution of the top ranked features for the
    \shap-score, and the KernelSHAP approximation
    from~\cite{DBLP:conf/nips/LundbergL17}. The distribution for the \shap-score
    is identical to the one in Figure~\ref{fig:top1-scores}.}
  \label{fig:top1-kshap}
\end{figure}

In Sec.~\ref{sec:rudin} we compared the \resp-score with the exact computation
of the \shap-score using the algorithm in Fig.\ref{fig:algorithm:shap}. In this
section, we also compare the scores with the the KernelSHAP
algorithm~\cite{DBLP:conf/nips/LundbergL17}, which approximates the \shap-score
by computing a linear model locally around the outcome. We evaluate the accuracy
and the performance difference of the approximation on the FICO dataset.

\begin{figure*}[t]
  \includegraphics[width=.5\textwidth]{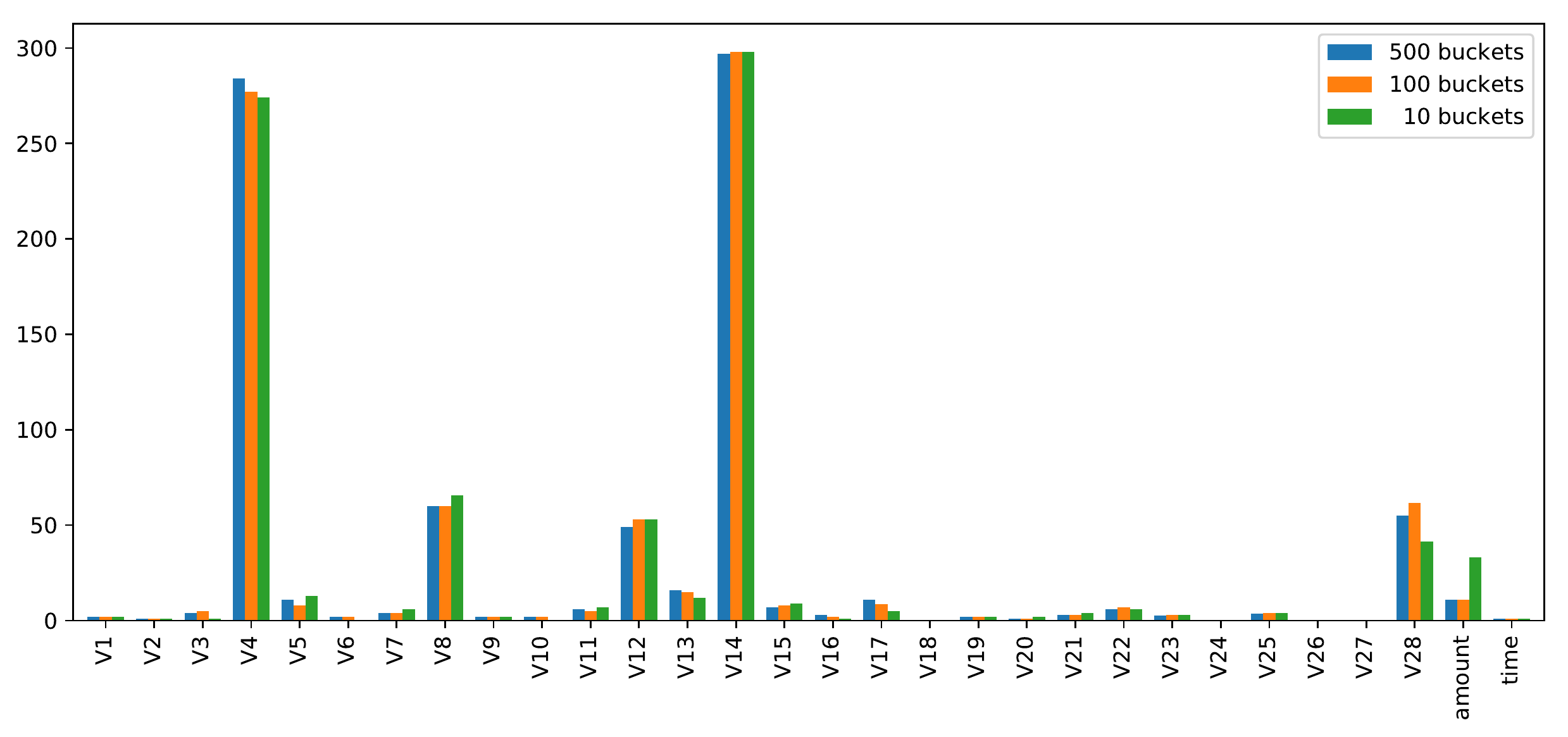}\hspace*{-0.1em}
  \includegraphics[width=.5\textwidth]{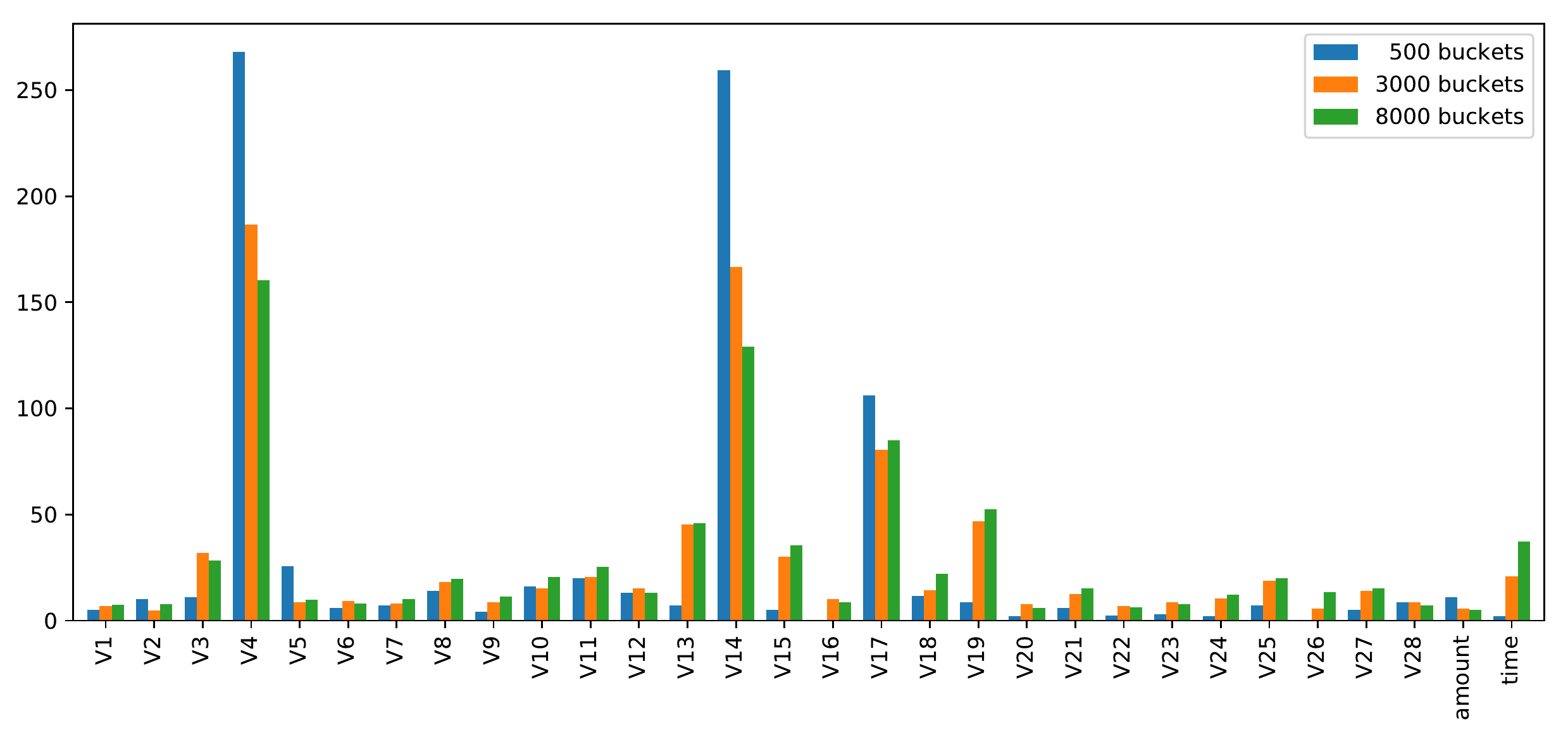}

  \vspace*{-1em}

  \caption{Distribution of the top ranked features for the (left) \resp-score,
    and (right) \shap-score where the features are bucketized into varying number
    of equi-depth buckets.}
  \label{fig:top1-creditcard}
\end{figure*}

Fig.\ref{fig:top1-kshap} shows the distribution for the top-1 explanation for
the \shap-score and the KernelSHAP approximation. We observe that the
distribution for the two scores is similar: For both scores,
\emph{ExternalRiskEstimate} and \emph{NetFractionRevolvingBurden} are the most
frequent top-1 explanations. Also, the most frequent explanation for the
\resp-score, \emph{MSinceMostRecentInqexcl7days}, is is not a frequent top-1
explanation for KernelSHAP. This shows that the KernelSHAP is more aligned with
the \shap-score, and the discussion on the differences
between the \resp\ and \shap-scores from Section~\ref{sec:rudin} also holds for
the KernelSHAP explanations.

The comparison of the top-4 explanation sets for \shap\ and KernelSHAP shows
that there is a significant tradeoff between the time it takes to compute the
explanation and the accuracy. On the one hand, the KernelSHAP explanation can be
computed significantly faster. In our experiments, KernelSHAP can approximate
the explanation scores for each entity in 0.78 seconds on average, whereas the
exact computation of the \shap-score using the algorithm in
Fig.\ref{fig:algorithm:shap} takes 19.8 seconds on average per entity.  The
comparison of the top-4 explanation sets for \shap\ and KernelSHAP, however,
shows that the scores have limited overlap due to the approximation in
KernelSHAP: For 58\% of the explained instances, the two top-4 explanation sets
overlap only for two or less explanations, and the Jaccard similarity
coefficient is 0.44.

In future work, we plan to investigate alternative approximations for the
\shap-scores, which can lead to performance improvements without significant
accuracy tradeoffs.

\section{Experiments with  the Kaggle Credit Card Fraud Dataset}
\label{appendix:creditcard}


We present our results for the evaluation of the \resp\ and \shap-scores on the
Kaggle Credit Card Fraud
dataset\footnote{https://www.kaggle.com/mlg-ulb/creditcardfraud},which is used
to detect fraudulent credit card transactions.

{\bf Dataset.}  The dataset consists of 284,807 credit card transactions. Each
transaction is described by 30 numerical input variables, out of which 28 are
normalized and anonymized for privacy reasons. The other two features are
\textit{Time} (seconds elapsed between the transaction and the first transaction
in the dataset), and \textit{Amount} (the amount of the transaction). The
dependent variable takes value 1 in case of fraud, and 0 otherwise. The
dependent variable in the dataset is highly unbalanced. Out of all transactions,
only 492 transactions are labeled as fraud.

We normalize \textit{Amount} and \textit{Time} with the scikit-learn
RobustScaler, and then separate the dataset into training and test data. The
test dataset is a random sample of 28,481 entities (10\% of the dataset).

{\bf Classification Model.} We use as classifier a logistic regression
model. Since the dependent variable in the dataset is highly
unbalanced, we trained the classifier over an undersample of the
training data, which has an equal distribution of positive and
negative outcomes.\footnote{See
  \url{https://www.kaggle.com/janiobachmann/credit-fraud-dealing-with-imbalanced-datasets}}
The classifier labels 846 entities in the original test dataset as
fraudulent, which means it misclassifies less than 3\% of the
entities. This results in an ROC-AUC score of 0.95 on the test
dataset.

{\bf Score Computation.}  All features in this dataset are high-precision
continuous variables, which leads to complications in the computation of the
\resp- and \shap-scores. For the \shap-score, the issue is that conditioning on
one feature returns a single entity, and, thus, no representative distribution
to compute the expected value over the outcome. For the \resp-score, the feature
domains are so large that the score becomes very expensive to compute,
especially for non-empty contingency sets.


To mitigate this issue, we bucketize the features into varying numbers of
equi-depth buckets. Each bucket is represented by the mean over the values that
fall into this bucket. Note that, there is a tradeoff between the number of
buckets and the time it takes to compute the score. This tradeoff is reversed
for the two scores. For the \resp-score, as the number of buckets
increases, the computation time increases. For the \shap-score, as the number of
buckets decreases, the computation time increases. We evaluate how sensitive the
two scores are to the number of buckets.

{\bf Evaluation Results.} Fig.~\ref{fig:top1-creditcard} presents the
distribution of the top explanations for the two scores and varying number of
buckets. For the case of 500 buckets, both scores frequently report the
variables $V_4$ and $V_{14}$ as top explanations. Beside these two features,
however, the distributions are not very comparable. For instance, the third most
common top explanation with the \shap-score is $V_{17}$, which is rarely a top
explanations for the \resp-score.

We observe that the distribution of the \resp-score is very consistent when the
number of buckets ranges between 10 and 500. The number of buckets do however
effect the time it takes to compute the score. On our commodity machine, we
computed the \resp-score for 10 buckets in less than 4 minutes (roughly 0.3
seconds per entity), while for 500 buckets it took around 11 hours. This is
mostly due to the computation for contingency sets of size 1, which is quadratic
in the number of features and domain sizes.  For the \shap-score, however, as
the number of buckets increases, the distribution over the features becomes more
uniform. This is because as we increase the number of buckets there are, the
number of entities after conditioning any feature decreases. The time to compute
the \shap-score ranged from 17 minutes for 8000 buckets to over 16 hours for 500
buckets. We attempted to compute the \shap-score for 100 buckets, but the
experiment took longer than the pre-defined timeout of 24 hours.  We conclude
that the \resp-score is robust to the bucketization of the features, while for
\shap-scores there is a tradeoff between the number of buckets and the ability
to attain good explanations.

We further compare the top-4 explanation sets for each entity for both scores
and 500 buckets. The majority of the entites (84\%) share one or two
explanations in the top-4 explanation sets of the two scores, while 10.5\% share
no explanation. Remarkably, for 75.4\% of the entities, the top-explanation for
the \resp-score is also one of the top-4 explanations for the \shap-score. On
the other hand, the top explanation in the \shap-score is one of the top-4
\resp-score explanations for only 67.4\% of the entities. Therefore, we conclude
that there is some overlap in the scores, but there is only a limited overlap
across the top-4 explanations.






\end{document}